\documentclass[preprint,12pt,authoryear]{elsarticle}
\usepackage[right=1.5cm,left=1.5cm,top=1.0cm,bottom=1.5cm]{geometry}
\usepackage{algorithm, algorithmic}  
\usepackage{amsmath, amssymb, amsfonts, amsthm}
\usepackage{hyperref}
\usepackage{mathptmx}
\usepackage{bm}
\usepackage{graphicx}
\usepackage{epstopdf}
\usepackage{subfigure}
\usepackage{caption}
\usepackage{soul}
\usepackage{tabularx}
\usepackage{booktabs}
\usepackage{multirow}
\usepackage{longtable}
\usepackage{color}
\usepackage{bbding} 
\usepackage{cleveref}
\usepackage[section]{placeins}
\biboptions{sort&compress}

\journal{Expert Systems with Applications}
\hypersetup{pdfauthor = whatever}
\newcolumntype{Z}{>{\centering\let\newline\\\arraybackslash\hspace{0pt}}X}

\renewcommand{\algorithmicrequire}{\textbf{Input:}}
\renewcommand{\algorithmicensure}{\textbf{Output:}}

\hypersetup{colorlinks = true}
\newtheorem{theorem}{Theorem}

\begin{document}
\captionsetup[figure]{labelfont={bf},name={Fig.},labelsep=period}
	\begin{frontmatter}
		\title{ Sparse $\epsilon$ insensitive zone bounded asymmetric elastic net  support vector machines for pattern classification}
		
		\author{Haiyan Du}
		\author{Hu Yang\corref{correspondingauthor}}\cortext[correspondingauthor]{Corresponding author. \emph{Email: yh@cqu.edu.cn}}
		
		\address{College of Mathematics and Statistics, Chongqing University, Chongqing, 401331, China}
		
	   \begin{abstract}
              Existing support vector machines(SVM) models are sensitive to noise and lack sparsity, which limits their performance. To address these issues, we combine the elastic net loss with a robust loss framework to construct a sparse $\varepsilon$-insensitive bounded asymmetric elastic net loss, and integrate it with SVM to build $\varepsilon$ Insensitive Zone Bounded Asymmetric Elastic Net Loss-based SVM($\varepsilon$-BAEN-SVM). $\varepsilon$-BAEN-SVM is both sparse and robust. Sparsity is proven by showing that samples inside the $\varepsilon$-insensitive band are not support vectors. Robustness is theoretically guaranteed because the influence function is bounded. To solve the non-convex optimization problem, we design a half-quadratic algorithm based on clipping dual coordinate descent. It transforms the problem into a series of weighted subproblems, improving computational efficiency via the $\varepsilon$ parameter. Experiments on simulated and real datasets show that $\varepsilon$-BAEN-SVM outperforms traditional and existing robust SVMs. It balances sparsity and robustness well in noisy environments. Statistical tests confirm its superiority. Under the Gaussian kernel, it achieves better accuracy and noise insensitivity, validating its effectiveness and practical value.

        \end{abstract}
		
		\begin{keyword}$\epsilon$ insensitive zone bounded asymmetric elastic net loss \sep  classification \sep Sparsity \sep Robustness \sep Half-quadratic algorithm
		\end{keyword}
	\end{frontmatter}
    
    \section{Introduction}
        Among the numerous machine learning algorithms, SVM\citep{1995Support} has become one of the most important tools due to its outstanding predictive performance and solid theoretical foundation, and is widely applied in fields such as image recognition \citep{wang2025,omran2026yixue1}, biomedicine \citep{li2025yiyue2}, financial forecasting\citep{kuo2024hybrid,tang2021robust} , and industrial inspection\citep{liu2026gongye1}. In machine learning, sparsity simplifies models and improves computational efficiency. It also helps us understand the underlying nature of the data. A key advantage of Hinge-SVM is the sparsity of its solution. This sparsity comes from the KKT conditions of the convex quadratic optimization problem that SVMs solve. A Lagrange multiplier is non‑zero only when a sample lies on the margin or violates the margin. This means that the SVM solution is determined entirely by these support vectors. The remaining samples do not participate in model construction. As a result, SVM achieves low memory usage and fast prediction speeds.

      However, some SVM variants inadvertently lose sparsity during optimization. For example, \cite{suykens1999least} proposed LS-SVM. It transforms the optimization problem into linear equations using a quadratic loss function. But this approach makes nearly all training samples become support vectors. Similarly, researchers who modify the loss function to improve SVM performance also unintentionally harm sparsity. \cite{PinSVM6604389} introduced the pinball loss to improve the noise robustness of Hinge-SVM. \cite{zhu2020support} combined pinball loss with Huber loss to handle its non‑differentiable points. They proposed HPSVM. All these loss functions cause every training sample to become a support vector. As a result, the model lacks sparsity.

     To restore sparsity in SVMs, researchers have proposed various sparsification strategies. These include support vector pruning\cite{xia2023dual} and cluster‑based pre‑selection\citep{han2016novel}. However, these methods often sacrifice some accuracy. \cite{vapnik1999overview} formally introduced the $\varepsilon$-insensitive loss function in 1995. Its core idea is to build an ``$\varepsilon$-insensitive band''. This band allows a small error between predicted and true values without any penalty. As a result, the model becomes more robust and also gains sparsity. Based on this idea, \citet{tian2013efficient} proposed a least‑squares support vector machine (LS‑SVM) based on the $\varepsilon$-insensitive loss. They used the parameter $\varepsilon$ to control model sparsity. This improves sparsity while keeping the computational simplicity of LS‑SVM. \citet{liu2016ramp} further combined the $\varepsilon$-insensitive loss with the ramp loss. They proposed a novel sparse ramp loss least‑squares support vector machine. \citet{huang2013support}combined the $\varepsilon$-insensitive loss with the pinball loss. This enhanced the sparsity of Pin‑SVM while maintaining its robustness to resampling. \cite{rastogi2018generalized} improved the traditional Pin‑SVM by proposing a generalized pinball loss. In their method, $\varepsilon_1$ and $\varepsilon_2$ are treated as optimization variables. They also add the term $c_1(\varepsilon_1 + \varepsilon_2)$ to the objective function. This allows the model to automatically learn the optimal insensitivity interval width from the data, which further improves sparsity.
     
        In addition to the studies on improving loss functions or performing variable selection, recent research has also explored SVM sparsity from the perspective of regularization. \cite{tian2023kernel} proposed the Sparse Support Vector Machine (SSVM). They replaced the $L_2$ penalty with an $L_1$ penalty. This allows the model to achieve both sample sparsity and feature sparsit. To efficiently solve sparse SVM models, several optimization algorithms have been widely adopted. These include coordinate descent\citep{wright2015coordinate}, the alternating direction method of multipliers \citep{zhang2012efficient}, and stochastic gradient descent \citep{bottou2012stochastic}. These algorithms can take full advantage of the sparse structure. They significantly reduce computational complexity. As a result, sparse SVMs show good scalability on large‑scale real‑world datasets.
        
         Existing robust support vector machines (SVMs) improve resistance to noise to some extent. However, when label noise and feature noise exist together, the robustness of the model still faces higher demands. In addition, many SVM models lack sparsity. Furthermore, optimizing non-convex bounded loss functions remains a challenge. Based on this background, we first use the advantage of the elastic net loss in slack variables. We combine the $\varepsilon$-insensitive loss with the $L_{aen}$ loss to propose the $L_{aen}^{\varepsilon}$ loss. This enhances the sparsity of EN-SVM. Then, we apply the RML framework \citep{fu2024RLM} to smooth $L_{aen}^{\varepsilon}$. This leads to an asymmetric elastic net loss function called $L_{\text{baen}}^{\varepsilon}$. We apply this loss to support vector classification to build a more robust classification model.
The main contributions of this paper are summarized as follows:
\begin{itemize}
            \item
 We propose the $\varepsilon$-insensitive bounded asymmetric elastic net loss function $L_{\text{baen}}^{\varepsilon}$, whose boundedness and asymmetry enable handling of both label and feature noise. Integrating this loss with support vector machines yields the $\varepsilon$-BAEN-SVM model, which jointly achieves robustness and sparsity.
 \item 
$\varepsilon$-BAEN-SVM is both sparse and robust. Theoretically, we prove that samples inside the $\varepsilon$-insensitive band will not become support vectors. This ensures the model's sparsity. We also derive the influence function of the model and prove that it is bounded. 
 \item 
 We address the non-convex optimization of $\varepsilon$-BAEN-SVM using a half-quadratic algorithm based on clipping dual coordinate descent. This method decomposes the non-convex problem into a series of iteratively weighted $\varepsilon$-insensitive asymmetric elastic net loss SVM subproblems, thereby clarifying the model’s robustness. Moreover, the choice of $\varepsilon$ substantially enhances computational efficiency.
 \item 
 We conduct numerical experiments on simulated and benchmark datasets. The results clearly show that compared with existing SVM methods, $\varepsilon$-BAEN-SVM achieves a better balance between sparsity and robustness under label noise and feature noise.
\end{itemize}

        The rest of the paper is organized as follows: Section 2 reviews recent related studies. In Section 3, we construct $\varepsilon$-BAEN-SVM and solve it using the clipDCD-based HQ algorithm. Next, we provide some theoretical analysis on $\varepsilon$-BAEN-SVM properties in Section 4. In Section 5, the results of artificial and benchmark datasets are utilized to confirm the effectiveness of $\varepsilon$-BAEN-SVM. Finally, Section 6 concludes the paper and discusses future research directions.

    \section{Related work}

In this section, we provide a brief review of related works. Let \(T=\{(x_{1},y_{1}),(x_{2},y_{2}),\ldots,(x_{n},y_{n})\}\) represent the set of training samples, where \(x_{i}\in\mathbb{R}^{p}\) is the \(i\)-th sample and \(y_{i}\in\{-1,+1\}\) is the corresponding label. The samples are organized into a data matrix \(X\in\mathbb{R}^{n\times p}\). Unless otherwise specified, all vectors are considered column vectors.

\subsection{$\varepsilon$-Insensitive Pinball Loss based Support Vector Machine}

To improve the sparsity of the standard Pinball Support Vector Machine (Pin-SVM), \citet{huang2013support} proposed an $\varepsilon$-insensitive version of the pinball loss, which is defined as
\begin{equation}
L_{\tau}^{\varepsilon}(u) = 
\begin{cases}
u - \varepsilon, & u > \varepsilon, \\[4pt]
0, & -\dfrac{\varepsilon}{\tau} \le u \le \varepsilon, \\[8pt]
-\tau\left(u + \dfrac{\varepsilon}{\tau}\right), & u < -\dfrac{\varepsilon}{\tau}.
\end{cases}
\end{equation}
where \(u = 1 - y_i\bigl(w^\top \phi(x_i) + b\bigr)\), \(\tau \ge 0\) controls the asymmetry, and \(\varepsilon \ge 0\) defines the size of the insensitive zone. Combining \(L_{\tau}^{\varepsilon}\) with SVM yields the $\varepsilon$-insensitive pinball loss SVM ($\varepsilon$-PinSVM). 

Introducing Lagrangian multipliers \(\alpha_i \ge 0\), \(\beta_i \ge 0\), \(\gamma_i \ge 0\), kernel \(K(x_i,x_j)=\phi(x_i)^\top \phi(x_j)\) and let \(\lambda_i = \alpha_i - \beta_i\), and applying the Karush–Kuhn–Tucker (KKT) conditions, we obtain the dual problem of \(\varepsilon\)-PinSVM.
\begin{align}
\max_{\lambda,\beta,\gamma} \quad & -\frac{1}{2}\sum_{i=1}^{m}\sum_{j=1}^{m} \lambda_i y_i K(x_i,x_j) y_j \lambda_j + \sum_{i=1}^{m} (\lambda_i + \varepsilon \gamma_i) \label{eq:dual}\\
\text{s.t.} \quad & \sum_{i=1}^{m} \lambda_i y_i = 0, \nonumber\\
& \lambda_i + \bigl(1+\tfrac{1}{\tau}\bigr)\beta_i + \gamma_i = C, \quad i=1,\dots,m, \nonumber\\
& \lambda_i + \beta_i \ge 0,\quad \beta_i \ge 0,\quad \gamma_i \ge 0, \quad i=1,\dots,m. \nonumber
\end{align}
The dual formulation shows that when \(\varepsilon>0\), many \(\lambda_i\) become zero, resulting in a sparse set of support vectors. 
At the same time, similar to the pinball loss, the \(\varepsilon\)-insensitive version maintains robustness to feature noise near the decision boundary.

        \subsection{Elastic net loss based-Support Vector Machine}
        \citet{qi2019ENSVM} put forward elastic net ($L_{en}$) loss, which imposes the elastic-net penalty to slack variables. By introducing $L_{en}$ loss into SVM, Qi proposed an elastic net loss-based SVM (ENSVM) expressed as
             \begin{equation}\begin{aligned}
             & \min_{w,b}\quad\frac{1}{2}\|w\|_{2}^{2}+\frac{c_{1}}{2}\xi^{\mathrm{T}}\xi+c_{2}e^{T}\xi ,\\& \mathrm{s.t.}\quad
             \begin{cases}e-DXw\leq \xi, \\\xi\geq\mathbf{0}, & \end{cases}
            \end{aligned}\end{equation}
            where \(D=\mathrm{diag}(y_1,y_2,\cdotp\cdotp\cdotp,y_n)\), $\xi=(\xi_1,\cdots,\xi_n)^T$ are the slack variables.  \citet{QI2022ENNHSVM} showed through the VTUB of ENNHSVM that the elastic net penalty has unique advantages for slack variables. Thus, improving the performance of  EN loss is very important.  

        To improve the ability of EN-SVM to handle feature noise, Qi designed the asymmetric elastic net ($L_{aen}$) loss \citep{qi2023CAENSVM} motivated by pinball loss as follows:
             \begin{equation}
             L_{aen}(z;p,\tau) =
            \begin{cases}
            \frac{p}{2} z^2 + (1-p) z, & z \geq 0 ,\\
            \tau \left( \frac{p\tau}{2} z^2 - (1-p) z \right), & z < 0,
            \end{cases}
            \label{eq:LAEN}
            \end{equation}
            where \( \tau \in [0,1] \) is derived from the pinball loss and \( p\in [0,1] \) governs the trade-off between the \( l_1 \) norm and the \( l_2 \) norm. According to \eqref{eq:LAEN}, \( L_{\text{aen}} \) like $L_{hinge}$ grows to infinity as \(\mathrm{z}\rightarrow\infty\), making it highly sensitive to outliers (label noise).
            
        \subsection{Robust Support Vector Machine}
         To mitigate the impact of label noise, bounded loss functions have been widely adopted due to their robustness. \citet{fu2024RLM} proposed a general framework of robust loss function for machine learning (RML), inspired by the Blinex loss. The framework is defined as
         \begin{equation}L(x)=\frac{1}{\lambda}\left(1-\frac{1}{1+b(x)\cdot h(x)}\right),\forall \lambda, b >0,\end{equation}
         where$h(z)$ denotes any unbounded loss function excluding linear forms, $\lambda$ is a scaling parameter that controls the upper bound of $\mathcal{L}(z)$ and the flatness of $h(z)$, and $b(z)$ is any non-negative function controlling the growth rate of the smoothed loss $\mathcal{L}(z)$. The RML framework can smoothly and adaptively bound any non-negative function and retain its inherently elegant properties, including symmetry, differentiability, and smoothness.

         Within the RML framework, \citet{ZHANG2024BQSVM} proposed the bounded quantile loss $L_{bq}$ to improve the robustness of Pin-SVM against label noise. The $L_{bq}$ loss is constructed by taking $h(x)=L_{pin}(x)$, which is formulated as
            \begin{equation}L_{bq}(z;\eta,\lambda,\tau)=\frac{1}{\lambda}(1-\frac{1}{1+\eta L_{pin}(z)}).\end{equation}
         Then \citet{ZHANG2024BQSVM} integrated $L_{bq}$ loss into SVM to obtain BQ-SVM. Its definition is as follows:
             \begin{equation}\min_{w,b}\frac{1}{2}(\|w\|_2^2+b^2)+\frac{C}{2}\sum_{i=1}^nL_{bq}(1-y_i(x_i^Tw+b)).\end{equation}
        
        Despite its robustness, the $L_{bq}$ loss remains non-differentiable at certain points, thereby increasing the complexity of the optimization process. To address this limitation, \citet{zhang2025BALSSVM} proposed the bounded least absolute squares loss $L_{bals}$ by setting $h(x)=L_{als}(x)$, which is formulated as:
            \begin{equation}L_{bals}(z;\eta,\lambda,\tau)=\frac{1}{\lambda}(1-\frac{1}{1+\eta L_{als}(z)}).\end{equation}
         Then \citet{zhang2025BALSSVM} combined $L_{bals}$ loss with SVM to obtain BALS-SVM. Its definition is as follows: \begin{equation}\min_{w,b}\frac{1}{2}(\|w\|_2^2+b^2)+\frac{C}{2}\sum_{i=1}^nL_{bals}(1-y_i(x_i^Tw+b)).\end{equation}
         However, However, BALS-SVM and BQ-SVM lack geometric advantage on slack variable, and nearly all samples are support vectors.
       
    \section{ $\varepsilon$ insensitive zone Bounded Asymmetric Elastic Net Loss-Based SVM} 
    \subsection{$\varepsilon$  insensitive zone Bounded Asymmetric Elastic Net Loss} 
       $L_{en}$ has the advantage of geometric rationality of slack variables, thereby enhancing the generalization ability of the model and possessing certain research significance. However, the $L_{en}$ and $L_{aen}$ are both convex loss functions and are sensitive to label noise. The $L_{baen}$ also causes the model to lose sparsity.

To improve the sparsity of SVM and the rationality of slack variables, this paper introduces an $\varepsilon$-insensitive band to improve $L_{aen}$, and proposes the asymmetric elastic net loss with an $\varepsilon$-insensitive band, denoted by $L_{aen}^{\varepsilon}$,
\begin{equation}
    L_{aen}^{\varepsilon}(z)=
\begin{cases}
\dfrac{p}{2}(z-\varepsilon)^2 + (1-p)(z-\varepsilon), & z>\varepsilon,\\[6pt]
0, & -\dfrac{\varepsilon}{\tau}\le z\le \varepsilon,\\[6pt]
\tau\left(\dfrac{p}{2}\left(z+\dfrac{\varepsilon}{\tau}\right)^2-(1-p)\left(z+\dfrac{\varepsilon}{\tau}\right)\right), & z<-\dfrac{\varepsilon}{\tau}.
\end{cases}
\end{equation}

Here, the parameter $\varepsilon>0$ is used to adjust the length of the insensitive band. $z=1-yw^Tx$, and $p,\tau\in(0,1)$ are tuning parameters.

To further improve its robustness to outliers, $L_{aen}^{\varepsilon}$ is smoothed again under the RML framework, and an innovative bounded asymmetric elastic net loss function with an $\varepsilon$-insensitive band, denoted by $L_{baen}^{\varepsilon}$, is proposed. The RML framework can preserve its asymmetry while making $L_{aen}^{\varepsilon}$ bounded. The specific expression of $L_{baen}^{\varepsilon}$ is as follows.
\begin{equation}
L_{baen}^{\varepsilon}(z;\lambda,\eta,\tau,p)=
\begin{cases}
\dfrac{1}{\lambda}\left(1-\dfrac{1}{1+\eta\left(\dfrac{p}{2}(z-\varepsilon)^2+(1-p)(z-\varepsilon)\right)}\right), & z>\varepsilon,\\[8pt]
0, & -\dfrac{\varepsilon}{\tau}\le z\le \varepsilon,\\[8pt]
\dfrac{1}{\lambda}\left(1-\dfrac{1}{1+\eta\tau\left(\dfrac{p}{2}\left(z+\dfrac{\varepsilon}{\tau}\right)^2-(1-p)\left(z+\dfrac{\varepsilon}{\tau}\right)\right)}\right), & z<-\dfrac{\varepsilon}{\tau}.
\end{cases}
\end{equation}

As shown in Fig\ref{fig:baen-svm fenjie}, the parameter $\lambda$ controls the upper bound of the function $L_{baen}^{\varepsilon}(z)$, while $\eta$ determines the smoothness of the loss curve. The larger the value of $\eta$, the faster the loss function reaches its maximum value. The parameter $\tau$ mainly controls the asymmetry of the loss function, thereby enhancing the robustness of the model to feature noise. The value of $p$ affects the sharpness of the curve and is closely related to the geometric characteristics of $\varepsilon$-BAEN-SVM. The detailed theoretical proof is given in Section~3.4.1. Therefore, the loss function $L_{baen}^{\varepsilon}$ possesses desirable properties such as boundedness and asymmetry, which can improve the robustness of the model.
         \begin{figure}[H]
		\centering
		\subfigure[\( L_{baen}^\epsilon \)  with different \(\lambda\) ($\eta=1,p=0.5,\tau=1,\epsilon=0.5$)]{
			\includegraphics[scale=0.5]{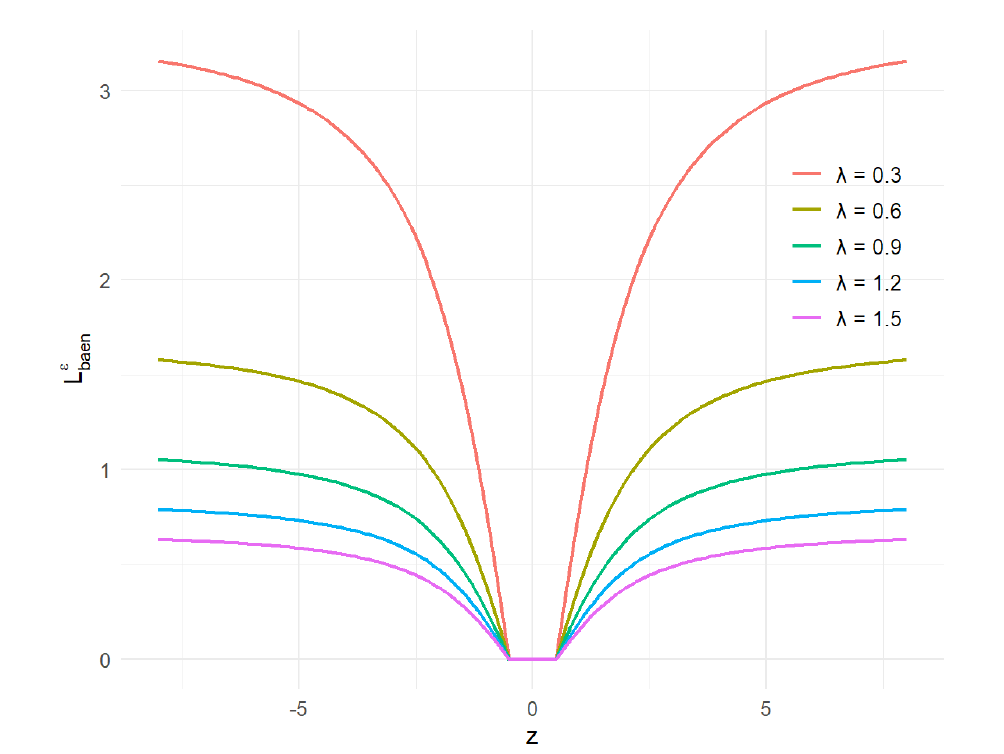}
		}
        	\subfigure[\( L_{baen}^\epsilon \)  with different \(\eta\) ($\lambda=1,p=0.5,\tau=1,\epsilon=0.5$)]{
			\includegraphics[scale=0.5]{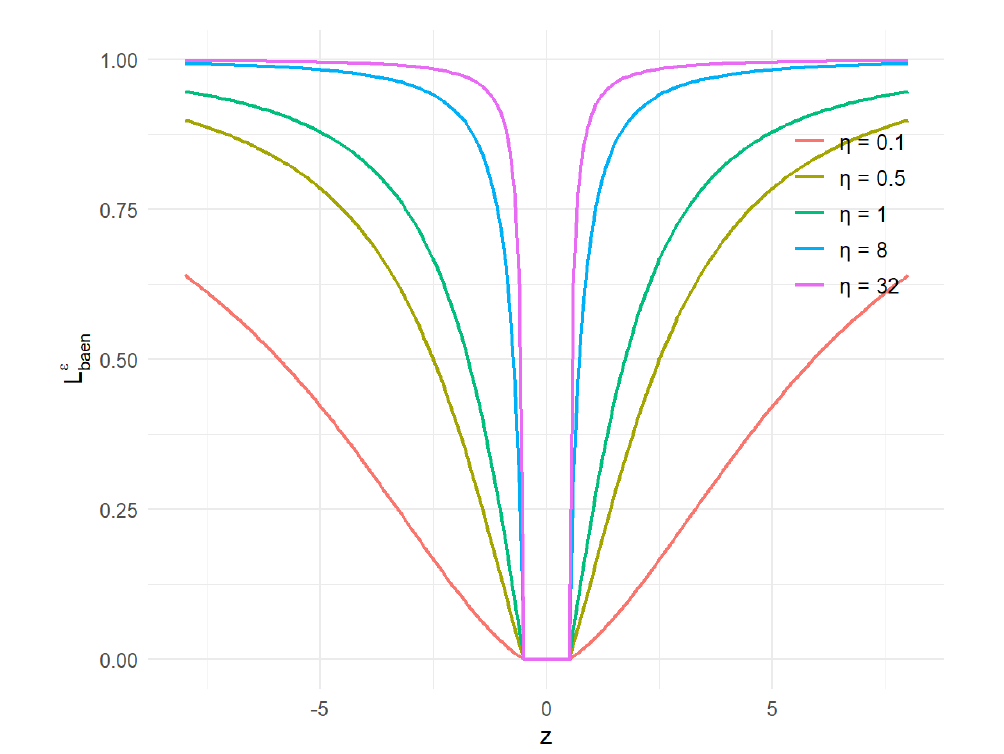}
		}
        \\
		\centering
		\subfigure[\( L_{baen}^\epsilon \)  with different \(\tau\) ($\lambda=1,\eta=1,p=0.5,\epsilon=0.5$)]{
			\includegraphics[scale=0.5]{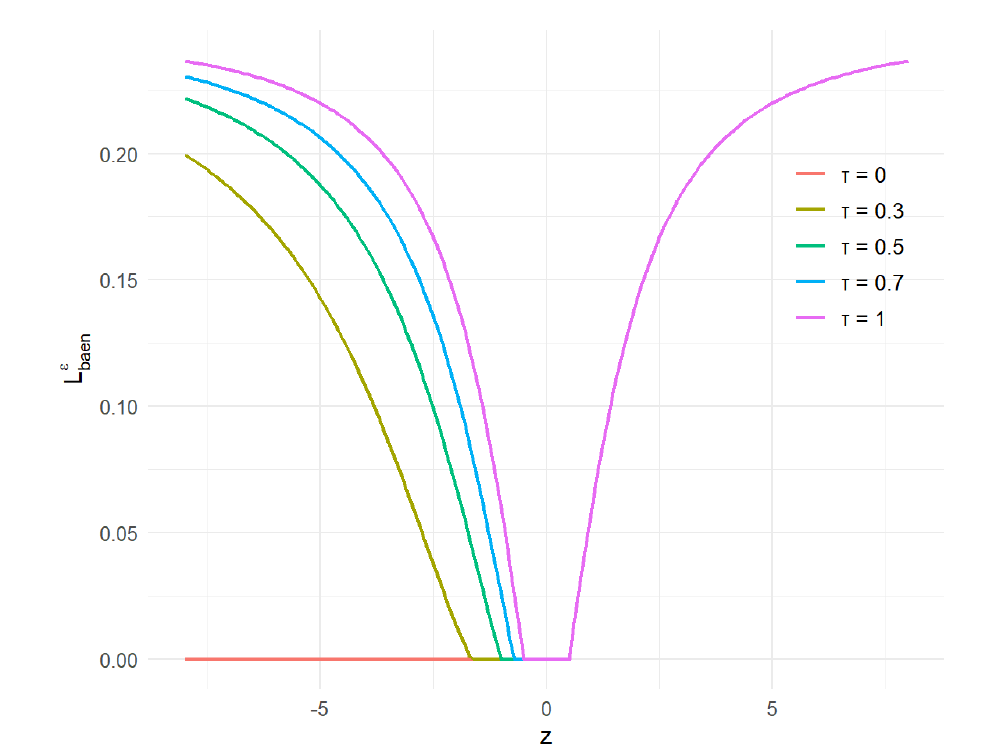}
		}
		\subfigure[\( L_{baen}^\epsilon \)with different $p$ ($\lambda=1,\eta=1,\tau=1,\epsilon=0.5$)]{
			\includegraphics[scale=0.5]{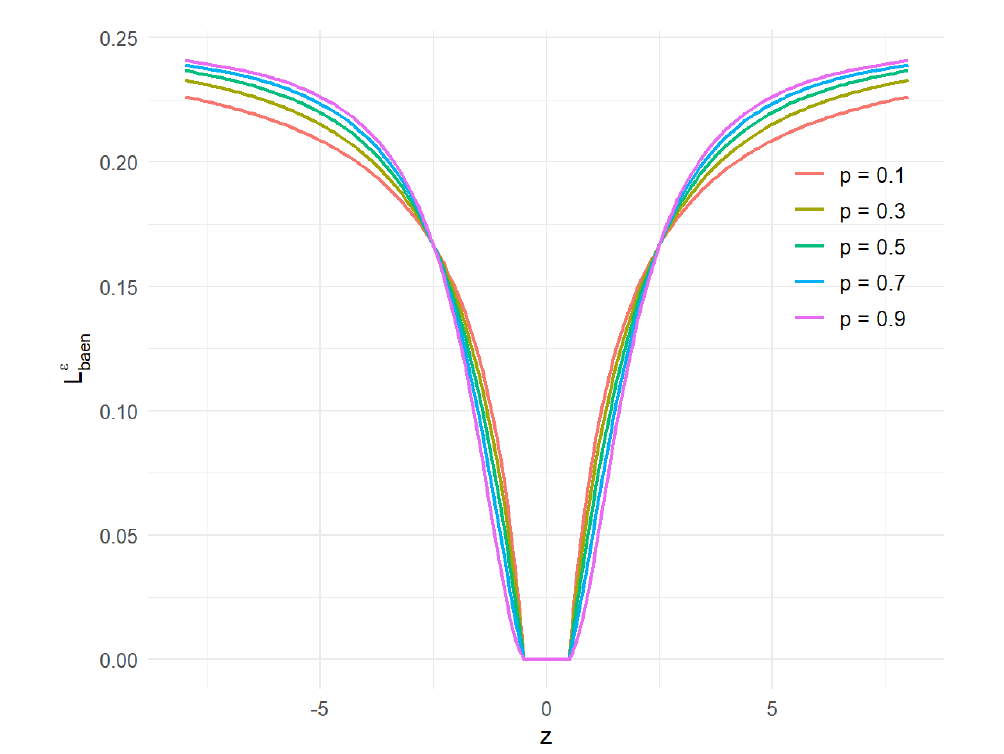}
		}
        \subfigure[\( L_{baen}^\epsilon \)with different $\epsilon$ ($\lambda=1,\eta=1,\tau=1,p=1$)]{
			\includegraphics[scale=0.5]{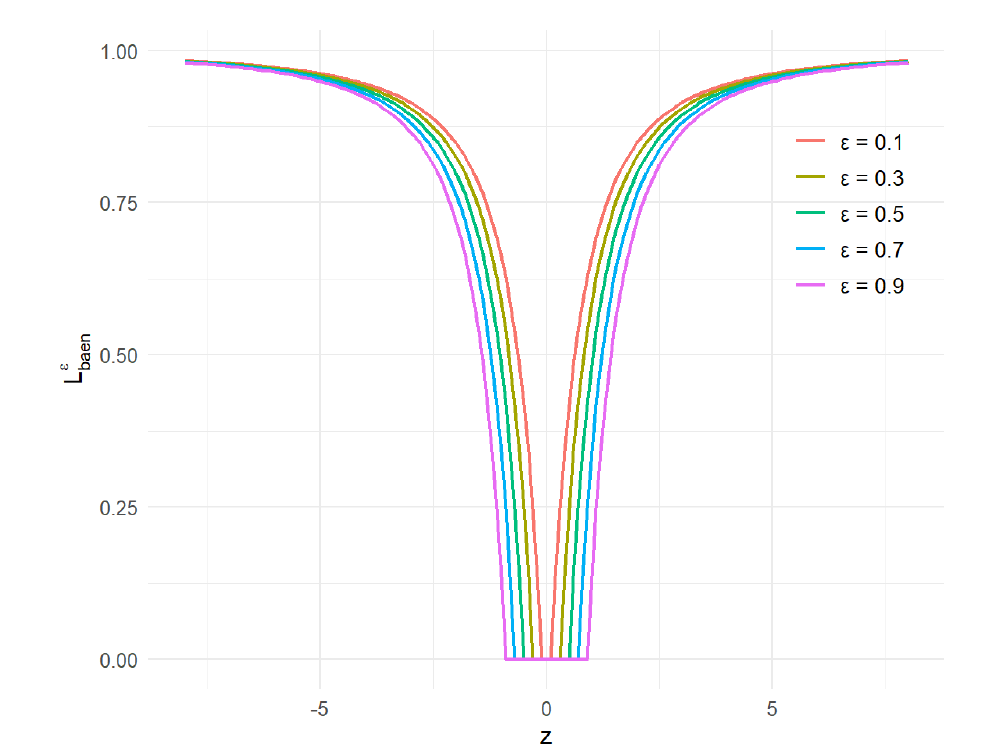}
		}
		\caption{Different parameter of \( L_{baen}^\epsilon\) } 
		\label{fig:baen-svm fenjie}
	\end{figure}

\begin{figure}[H]  
    \centering
    \includegraphics[width=0.9\textwidth]{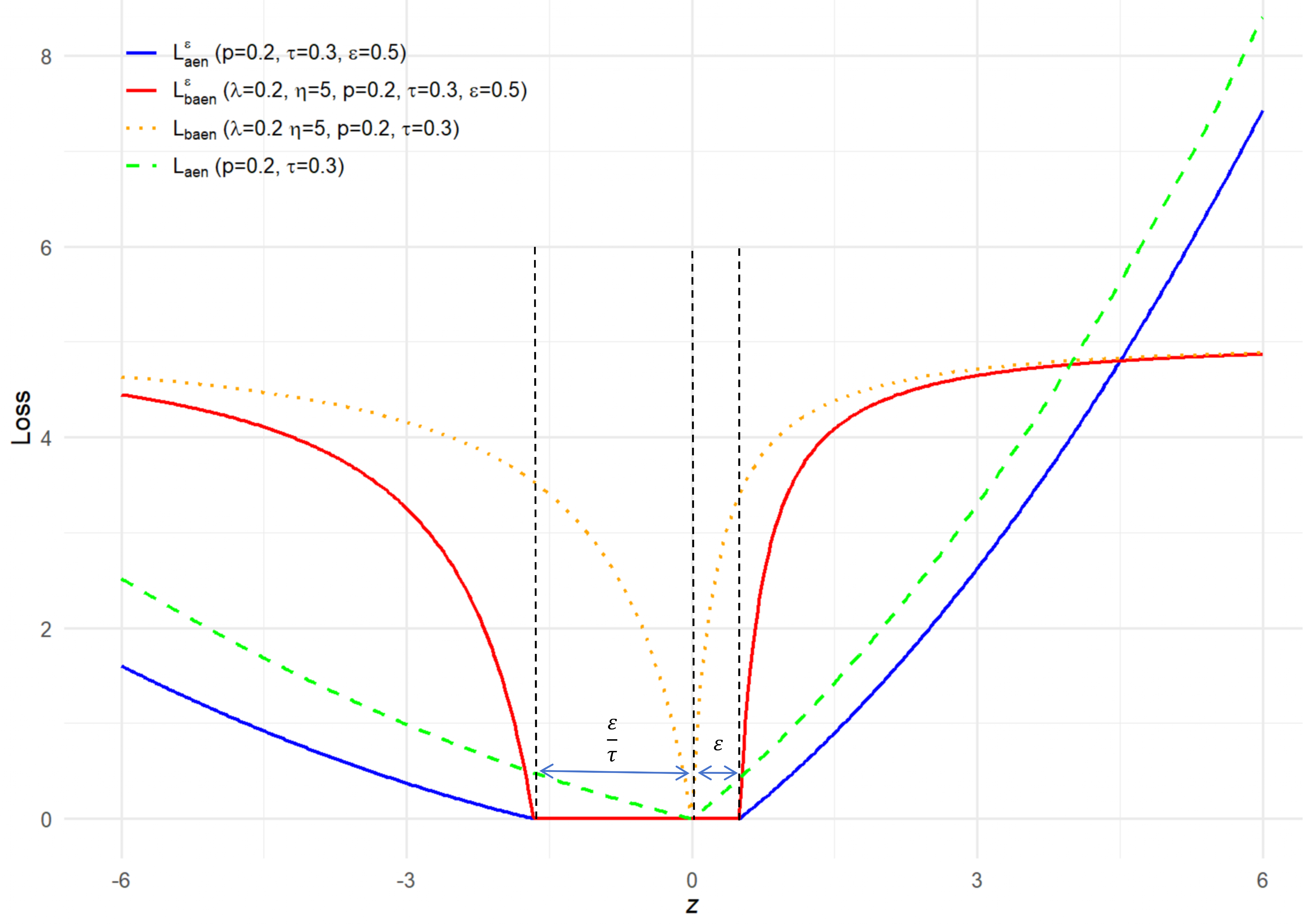} 
    \caption{Comparison of $L_{\mathrm{baen}}^{\varepsilon}$, $L_{\mathrm{baen}}$, $L_{\mathrm{aen}}^{\varepsilon}$ and $L_{\mathrm{aen}}$ losses.}
    \label{fig:loss compare}
\end{figure}

Let $\varepsilon=0.5$ and $\tau=0.3$. Then the loss function curves of $L_{baen}^{\varepsilon}$, $L_{baen}$, $L_{aen}^{\varepsilon}$, and $L_{aen}$ can be obtained, as shown in Fig.~3.2. Compared with $L_{baen}$, the loss $L_{baen}^{\varepsilon}$ mainly improves model sparsity by introducing the $\varepsilon$-insensitive band. When $z\in(-\varepsilon/\tau,\varepsilon)=(-1.6,0.5)$, the value of $L_{baen}^{\varepsilon}$ is $0$, that is, samples in this region produce no loss, and the corresponding dual variables are zero, thereby ensuring the sparsity of the model. In contrast, the loss value of $L_{baen}$ is zero only when $z=0$, and thus it lacks sparsity. The detailed theoretical proof is given in Section~3.4.2.

Compared with $L_{aen}$ and $L_{aen}^{\varepsilon}$, the loss $L_{baen}^{\varepsilon}$ is nonconvex and bounded, and is less sensitive to label noise. Specifically, since $\eta>0$ and $L_{aen}(z)$ increases monotonically with respect to $z$, we have
\begin{equation}
\lim_{z\to\infty} L_{baen}^{\varepsilon}(z)
=
\lim_{z\to\infty}\frac{1}{\lambda}\cdot\frac{\eta L_{aen}^{\varepsilon}(z)}{1+\eta L_{aen}^{\varepsilon}(z)}
=
\frac{1}{\lambda}\lim_{\eta L_{aen}^{\varepsilon}(z)\to\infty}
\frac{1}{\left(\dfrac{1}{\eta L_{aen}^{\varepsilon}(z)}+1\right)}
=
\frac{1}{\lambda}.
\end{equation}

This indicates that the upper bound of the function $L_{baen}^{\varepsilon}$ is $1/\lambda$. However, the nonconvexity of the loss function usually leads to difficulties in numerical computation. Therefore, it is necessary to design an efficient optimization algorithm for solving it.
        
        \subsection{The $\epsilon$-BAEN-SVM Model}
       
        The standard SVM adopts the $L_{hinge}$ loss function, which is not only insensitive to noise, but also suffers from the limitation of geometrically unreasonable slack variables. The $L_{baen}^{\varepsilon}$ constructed in the previous section can not only preserve boundedness and asymmetry, thereby simultaneously addressing label noise and feature noise, but also remedy the lack of sparsity in $L_{baen}$. Therefore, in this section, we apply the $L_{baen}^{\varepsilon}$ loss function to the traditional SVM and propose the robust support vector machine model with the $\varepsilon$-insensitive bounded asymmetric elastic net loss, namely the $\varepsilon$-BAEN-SVM model.

Consider the following binary classification problem with $n$ training samples and $p$ features. Let
\begin{equation}
    T=\{(x_1,y_1),(x_2,y_2),\cdots,(x_n,y_n)\}
\end{equation}

denote the dataset in the given feature space, where $x_i\in\mathbb{R}^p$ is the $i$-th sample, and $y_i\in\{-1,+1\}$ is the corresponding class label. All samples form the data matrix $X\in\mathbb{R}^{n\times p}$. Then, in the linear case, the $\varepsilon$-BAEN-SVM can be expressed as
\begin{equation}
\min_{w,b}\ \frac{1}{2}\left(\|w\|^2+b^2\right)+C\sum_{i=1}^{n}L_{baen}^{\varepsilon}\!\left(1-y_i(w^Tx_i+b)\right).
\end{equation}

Here, $C>0$ is a tuning parameter, $w\in\mathbb{R}^{p\times 1}$ is the normal vector of the hyperplane, and $b$ is the intercept term. Since the intercept term $b$ can be absorbed into the normal vector $w$, we can derive $\tilde{x}_i=(x_i^T,1)$ and $\tilde{w}=(w^T,b)^T$. Meanwhile, $\lambda$ can also be absorbed into $C$. Therefore, we may fix $\lambda=1$, and denote
\begin{equation}
L_{baen1}^{\varepsilon}(z;\eta,p,\tau)=L_{baen}^{\varepsilon}(z;1,\eta,p,\tau).
\end{equation}
Then the model can be simplified as
\begin{equation}\label{eq:18}
\min_{\tilde{w}}\ \frac{1}{2}\|\tilde{w}\|^2+C\sum_{i=1}^{n}L_{baen1}^{\varepsilon}\!\left(1-y_i\tilde{w}^T\tilde{x}_i;\eta,p,\tau\right).
\end{equation}

In this formulation, $\frac{1}{2}\|\tilde{w}\|^2$ is the regularization term used to measure model complexity, and
\begin{equation}
\sum_{i=1}^{n}L_{baen1}^{\varepsilon}\!\left(1-y_i\tilde{w}^T\tilde{x}_i;\eta,p,\tau\right)
\end{equation}
represents the $\varepsilon$-insensitive bounded asymmetric elastic net loss, which is used to measure the empirical risk. $C>0$ is a tuning parameter.
            
        \subsection{The clipDCD-based HQ Algorithm for $\epsilon$-BAEN-SVM}
Since the ADMM algorithm does not involve inner product operations when solving the coefficient vector during computation, it increases the difficulty of solving the nonlinear $\varepsilon$-BAEN-SVM. Therefore, inspired by the work of \cite{zhang2025BALSSVM}, we design a Half Quadratic based Clipping Dual Coordinate Descent algorithm, abbreviated as HQ-ClipDCD, to solve the nonlinear $\varepsilon$-BAEN-SVM model.

Through straightforward simplification, the original optimization problem \eqref{eq:18} can be equivalently written as
\begin{equation}
\label{eq:44}
\max_{\tilde{w}}
-\frac{1}{2}\|\tilde{w}\|^2
+
C\sum_{i=1}^{n}
\frac{1}{1+\eta L_{aen}^{\varepsilon}\bigl(1-y_i\tilde{w}^{T}\phi(\tilde{x}_i)\bigr)}.
\end{equation}

 we can further rewrite the objective function in (\ref{eq:44}) as
\begin{equation}
\label{eq:45}
\begin{aligned}
&-\frac{1}{2}\|\tilde{w}\|^2
+ C\sum_{i=1}^{n}
\sup_{v_i<0}
\left\{
\eta L_{aen}^{\varepsilon}\bigl(1-y_i\tilde{w}^{T}\phi(\tilde{x}_i)\bigr)v_i-g(v_i)
\right\} \\
&=
\sup_{v<0}
\left\{
-\frac{1}{2}\|\tilde{w}\|^2
+ C\sum_{i=1}^{n}
\left(
\eta L_{aen}^{\varepsilon}\bigl(1-y_i\tilde{w}^{T}\phi(\tilde{x}_i)\bigr)v_i-g(v_i)
\right)
\right\}.
\end{aligned}
\end{equation}

According to (\ref{eq:45}), it follows that (\ref{eq:44}) is equivalent to
\begin{equation}
\label{eq:46}
\max_{\tilde{w},\,v<0}
-\frac{1}{2}\|\tilde{w}\|^2
+ C\sum_{i=1}^{n}
\left(
\lambda L_{aen}^{\varepsilon}\bigl(1-y_i\tilde{w}^{T}\phi(\tilde{x}_i)\bigr)v_i-g(v_i)
\right).
\end{equation}

Next, an alternating iterative algorithm is designed to solve (\ref{eq:46}). In brief, the algorithm can be summarized as follows: first optimize \(v\) for a given \(\tilde{w}\), and then optimize \(\tilde{w}\) for a given \(v\). Specifically, suppose that \(\tilde{w}\) is given at the \(s\)-th iteration, then (\ref{eq:46}) is equivalent to
\begin{equation}
\label{eq:47}
\max_{v^s<0}\sum_{i=1}^{n}
\left(
\lambda L_{aen}^{\varepsilon}\bigl(1-y_i\tilde{w}^{T}\phi(\tilde{x}_i)\bigr)v_i^s-g(v_i^s)
\right).
\end{equation}

 we update
\begin{equation}
\label{eq:48}
v^s=-\frac{1}{\left(1+\lambda L_{aen}^{\varepsilon}\bigl(1-y_i\tilde{w}^{T}\phi(\tilde{x}_i)\bigr)\right)^2}<0.
\end{equation}

Then, by fixing \(v\) as \(v^s\), we update \(\tilde{w}^s\) through
\begin{equation}
\label{eq:49}
\tilde{w}^{\,s}
=
\arg\min_{\tilde{w}}
\frac{1}{2}\|\tilde{w}\|^2
+
C\sum_{i=1}^{n}
L_{aen}^{\varepsilon}\bigl(1-y_i\tilde{w}^{T}\phi(\tilde{x}_i)\bigr)(-v_i).
\end{equation}

Define \(\omega_i=C\lambda(-v_i)\). Then (\ref{eq:49}) can be rewritten as a support vector machine based on the weighted asymmetric elastic net loss with an \(\varepsilon\)-insensitive band, namely the \(\varepsilon\)-AEN-WSVM:
\begin{equation}
\label{eq:410}
\min_{\tilde{w}} \frac{1}{2}\|\tilde{w}\|^2+\sum_{i=1}^{n}\omega_iL_{aen}^{\varepsilon}\!\left(1-y_i\tilde{w}^{T}\phi(\tilde{x}_i)\right).
\end{equation}

Let \(X=(x_1^T,x_2^T,\cdots,x_n^T)^T,
D=\operatorname{diag}(y_1,y_2,\cdots,y_n),
\Omega=\operatorname{diag}(\omega_1,\omega_2,\cdots,\omega_n).
\)
Then (\ref{eq:410}) can be written in matrix form as
\begin{equation}
\label{eq:411}
\begin{aligned}
\min_{\tilde{w},\xi_+,\xi_-}\quad &
\frac{1}{2}\|\tilde{w}\|^2
+\frac{p}{2}\xi_+^T\Omega\xi_+
+(1-p)e^T\Omega\xi_+
+\frac{p\tau}{2}\xi_-^T\Omega\xi_- +\tau(1-p)e^T\Omega\xi_- \\
\text{s.t.}\quad &
e-DA\tilde{w}\le \xi_+ + \varepsilon, DA\tilde{w}-e\le \xi_- + \frac{\varepsilon}{\tau}.
\end{aligned}
\end{equation}

Here,
\begin{equation}
\label{eq:411a}
A=\bigl(\phi(\tilde{x}_1)^T,\phi(\tilde{x}_2)^T,\cdots,\phi(\tilde{x}_n)^T\bigr)^T.
\end{equation}
By introducing the Lagrange multiplier vectors \(\alpha\ge 0\) and \(\beta\ge 0\), the Lagrangian function is defined as
\begin{equation}
\label{eq:412}
\begin{aligned}
L(\tilde{w},\xi_+,\xi_-,\alpha,\beta)
={}&
\frac{1}{2}\|\tilde{w}\|^2
+\frac{p}{2}\xi_+^T\Omega\xi_+
+(1-p)e^T\Omega\xi_+
+\frac{p\tau}{2}\xi_-^T\Omega\xi_- +\tau(1-p)e^T\Omega\xi_-\\
&
+\alpha^T(e-DA\tilde{w}-\xi_+-\varepsilon)+\beta^T\!\left(DA\tilde{w}-e-\xi_- -\frac{\varepsilon}{\tau}\right).
\end{aligned}
\end{equation}

According to the KKT conditions, setting the partial derivatives of the Lagrangian function in \ref{eq:412} with respect to \(\tilde{w}\), \(\xi_+\), and \(\xi_-\) to zero yields
\begin{equation}
\label{eq:413}
\begin{cases}
\dfrac{\partial L}{\partial \tilde{w}}
=\tilde{w}-A^TD\alpha+A^TD\beta=0,\\[6pt]
\dfrac{\partial L}{\partial \xi_+}
=p\Omega\xi_+ +(1-p)\Omega e-\alpha=0,\\[6pt]
\dfrac{\partial L}{\partial \xi_-}
=p\tau\Omega\xi_- +\tau(1-p)\Omega e-\beta=0.
\end{cases}
\end{equation}

The complementary slackness conditions are
\begin{equation}
\label{eq:414}
\begin{cases}
\alpha^T(e-DA\tilde{w}-\xi_+-\varepsilon)=0,\\[6pt]
\beta^T\left(DA\tilde{w}-e-\xi_- -\dfrac{\varepsilon}{\tau}\right)=0.
\end{cases}
\end{equation}

Substituting (\ref{eq:413}) into (\ref{eq:412}), we obtain
\begin{equation}
\label{eq:416}
\begin{aligned}
L(\tilde{w},\xi,\alpha,\beta)
={}&
-\frac{1}{2}(\alpha-\beta)^T DXX^T D(\alpha-\beta)
-\frac{1}{2p}\alpha^T\Omega^{-1}\alpha +\frac{\tau-2}{2p}\beta^T\Omega^{-1}\beta\\
&+\frac{\tau(1-p)(2-\tau)}{p}e^T\beta
+(\alpha-\beta)^T e-\left(\alpha+\frac{\beta}{\tau}\right)^T\varepsilon
+\frac{1-p}{p}\alpha^T e .
\end{aligned}
\end{equation}

Therefore, the dual problem of (\ref{eq:45}) can be written as
\begin{equation}
\label{eq:417}
\begin{aligned}
\min_{\alpha,\beta}\quad
&
\frac{1}{2}(\alpha-\beta)^T DXX^T D(\alpha-\beta)
+\frac{1}{2p}\alpha^T\Omega^{-1}\alpha
-\frac{\tau-2}{2p}\beta^T\Omega^{-1}\beta \\
&
-\frac{\tau(1-p)(2-\tau)}{p}\beta^T e
-(\alpha-\beta)^T e
+\left(\alpha+\frac{\beta}{\tau}\right)^T\varepsilon
-\frac{1-p}{p}\alpha^T e \\
\text{s.t.}\quad & \alpha,\beta \ge 0.
\end{aligned}
\end{equation}

Let
\begin{equation}
\label{eq:417a}
u=
\begin{pmatrix}
\alpha\\
\beta
\end{pmatrix},
q=
\begin{pmatrix}
e+\dfrac{1-p}{p}e-\varepsilon\\[6pt]
-e+\dfrac{\tau(1-p)(2-\tau)}{p}e-\dfrac{\varepsilon}{\tau}
\end{pmatrix},
H=
\begin{pmatrix}
DXX^T D+\dfrac{1}{p}\Omega^{-1} & -DXX^T D\\[6pt]
-DXX^T D & DXX^T D-\dfrac{\tau-2}{p}\Omega^{-1}
\end{pmatrix}.
\end{equation}

Then (\ref{eq:417}) can be rewritten as the following quadratic programming problem:
\begin{equation}
\label{eq:418}
\begin{aligned}
\min_{u}\quad & \frac{1}{2}u^T H u-q^T u \\
\text{s.t.}\quad & 0 \le u .
\end{aligned}
\end{equation}

Finally, we employ the clipping dual coordinate descent algorithm (ClipDCD) to solve (\ref{eq:418}). The HQ-ClipDCD based solution framework for nonlinear $\varepsilon$-BAEN-SVM is shown in Algorithm~4.1.
Finally, we employ the clipping dual coordinate descent algorithm (ClipDCD) to solve (\ref{eq:418}). The HQ-ClipDCD based solution framework for nonlinear $\varepsilon$-BAEN-SVM is shown in Algorithm~\ref{alg:hqclipdcd}.

\begin{algorithm}[htbp]
\caption{HQ-ClipDCD for solving nonlinear $\varepsilon$-BAEN-SVM}
\label{alg:hqclipdcd}
\renewcommand{\algorithmicrequire}{\textbf{Input:}}
\renewcommand{\algorithmicensure}{\textbf{Output:}}
\begin{algorithmic}[1]
\REQUIRE Initial values $v^0,\alpha^0,\beta^0,u^0$, $s=0$; training set $D=\{(\tilde{x}_i,y_i)\}_{i=1}^{n}$ with $\tilde{x}_i\in\mathbb{R}^{(p+1)\times 1}$; maximum number of iterations $h_{\max}>0$; and $\epsilon\in\mathbb{R}^{+}$.
\ENSURE The optimal solution of (\ref{eq:418}).
\WHILE{$s\le h_{\max}$}
    \STATE Compute $\Omega=\operatorname{diag}(-C\lambda v^s)$.
    \STATE Solve subproblem \ref{eq:418} by the ClipDCD algorithm and obtain $u^{s+1}=\left(((\alpha^{s+1})^{T},(\beta^{s+1})^{T})\right)^{T}$.
    \IF{$\|u^{s+1}-u^{s}\|_{2}<\epsilon$}
        \STATE Break.
    \ELSE
        \STATE Update $v^s$ by (\ref{eq:48}).
        \STATE Let $s=s+1$.
    \ENDIF
\ENDWHILE
\STATE Return $\alpha^{*}=\alpha^{s}$ and $\beta^{*}=\beta^{s}$.
\end{algorithmic}
\end{algorithm}

By Algorithm~\ref{alg:hqclipdcd}, after obtaining $\alpha_i^{*}$ and $\beta_i^{*}$, the final decision function of nonlinear $\varepsilon$-BAEN-SVM can be written as
\begin{equation}
\label{eq:419}
f(x)=\sum_{i=1}^{n} y_i\,k(\tilde{x},\tilde{x}_i)\bigl(\alpha_i^{*}-\beta_i^{*}\bigr).
\end{equation}

    \section{Properties of $\varepsilon$-BAEN-SVM} \label{s4}
    This section analyzes the main properties of our proposed $\varepsilon$-BAEN-SVM, encompassing sparsity, noise insensitivity, and computational complexity. 
        \subsection{Sparsity}\label{s4.1}
       
       The sparsity of SVM means that, after training is completed, the decision function of the final model depends only on a subset of the training samples, which are called support vectors. Most training samples do not contribute directly to the decision boundary of the model and therefore can be ignored. This property gives support vector machines significant advantages in computational efficiency and storage requirements. In the original formulations of EN-SVM and BAEN-SVM, however, most training samples contribute directly to the decision function. Therefore, we introduce an $\varepsilon$-insensitive band to improve $L_{aen}$, and accordingly propose the $\varepsilon$-BAEN-SVM. This makes $\varepsilon$-BAEN-SVM sparser than BAEN-SVM. Next, we mainly prove the sparsity of $\varepsilon$-BAEN-SVM.

For a sample $x_i$, according to the complementary slackness conditions of the dual problem of $\varepsilon$-BAEN-SVM, we have
\begin{equation}
\label{eq:324}
\begin{cases}
\alpha_i\bigl(1-y_i x_i^{T}w-\xi_{+i}-\varepsilon\bigr)=0,\\[4pt]
\beta_i\bigl(y_i x_i^{T}w-1-\xi_{-i}-\dfrac{\varepsilon}{\tau}\bigr)=0.
\end{cases}
\end{equation}

Here, $\alpha_i,\beta_i\geq 0$ are Lagrange multipliers, $\xi_{+i}$ and $\xi_{-i}$ are the slack variables corresponding to sample $x_i$, and $\varepsilon>0$ and $\tau\in(0,1)$ are tuning parameters.

When $0<1-y_i x_i^{T}w<\varepsilon$, we have $\xi_{+i}=\xi_{-i}=0$. Then, it follows that
\(
-\varepsilon < 1-y_i x_i^{T}w-\xi_{+i}-\varepsilon < 0.
\)
Furthermore, from the complementary slackness conditions in (\ref{eq:324}), we obtain $\alpha_i=0$. Similarly, we have
\(
-\varepsilon-\frac{\varepsilon}{\tau}
< y_i x_i^{T}w-1-\xi_{-i}-\frac{\varepsilon}{\tau}
< -\frac{\varepsilon}{\tau},
\)
and from the complementary slackness conditions in (\ref{eq:324}), we obtain $\beta_i=0$.

When $-\dfrac{\varepsilon}{\tau}<1-y_i x_i^{T}w<0$, we also have $\xi_{+i}=\xi_{-i}=0$. Then, it follows that
\(
-\varepsilon-\frac{\varepsilon}{\tau}
< 1-y_i x_i^{T}w-\xi_{+i}-\varepsilon
< -\varepsilon.
\)
Furthermore, from the complementary slackness conditions in (\ref{eq:324}), we obtain $\alpha_i=0$. Similarly, we have
\(
-\frac{\varepsilon}{\tau}
< y_i x_i^{T}w-1-\xi_{-i}-\frac{\varepsilon}{\tau}
< 0,
\)
and from the complementary slackness conditions in (\ref{eq:324}), we obtain $\beta_i=0$.

In summary, when $1-y_i x_i^{T}w$ lies in the interval $\left(-\dfrac{\varepsilon}{\tau},\varepsilon\right)$, we have $\alpha_i=\beta_i=0$. Therefore, it can be concluded that $\varepsilon$-BAEN-SVM possesses sparsity.

In particular, when $\varepsilon=0$, $\varepsilon$-BAEN-SVM degenerates into BAEN-SVM. In this case, only when $1-y_i x_i^{T}w=0$ does $\alpha_i=\beta_i=0$ hold. This indicates that BAEN-SVM does not possess sparsity.

        \subsection{Noise Insensitivity}\label{s4.3}
From the construction of the loss function $L_{baen}^{\varepsilon}$, it can be seen that it not only possesses boundedness, which ensures robustness to label noise (outliers), but also inherits the insensitivity of $L_{aen}$ to feature noise. Therefore, in this section, the noise insensitivity of $\varepsilon$-BAEN-SVM is investigated from two aspects, namely label noise and feature noise.

 \subsubsection{Robust to Label Noise}

For robustness to label noise, we prove this property by showing that the influence function is bounded. This concept was first introduced by Hampel\citep{Hampel1974IFcurve}. The influence function measures the stability of an estimator under infinitesimal contamination. A robust estimator should have a bounded influence function\citep{Wang2013ESL}. Before presenting the main results, we first make some reasonable assumptions on the distribution of the training data.

Let the probability distribution of the sample point $(x_0^T,y_0)^T$ be denoted by $p_0$. Let $(x^T,y)^T\in\mathbb{R}^{p+1}$ be drawn from the probability distribution $F$. The contaminated distribution of $F$ and $p_0$ is defined as $F_{\theta}=(1-\theta)F+\theta p_0$, where $\theta\in(0,1)$ is a mixing proportion. For a given parameter, let the solution obtained under the contaminated distribution $F_{\theta}$ be denoted by $w_{\theta}^{*}$, and let the solution obtained under the distribution $F$ be denoted by $w_{0}^{*}$, where
\begin{equation}
\label{eq:325}
\begin{cases}
w_0^*=\arg\min\limits_{w}\left[\dfrac{1}{2}\|w\|^2+nC\displaystyle\int L_{baen}^{\varepsilon}\,dF\right],\\[8pt]
w_{\theta}^*=\arg\min\limits_{w}\left[\dfrac{1}{2}\|w\|^2+nC\displaystyle\int L_{baen}^{\varepsilon}\,dF_{\theta}\right].
\end{cases}
\end{equation}

The influence function of the sample point $(x_0^T,y_0)^T$ is defined as
\begin{equation}
\label{eq:326}
\mathrm{IF}(x_0,y_0;w_0^*)=\lim_{\theta\to 0^+}\frac{w_{\theta}^*-w_0^*}{\theta}.
\end{equation}

Before presenting the results, we make the following common assumptions on the distribution of the training data.

\textbf{Assumption 3:} The second moment of the random variable $x\in X$ exists, that is, $E\|x\|^2<\infty$.

\textbf{Assumption 4:}\(W_0=\left(\frac{1}{nC}I+\int \frac{\partial z(x_0,y_0,w_0^*)}{\partial w_0^*}\frac{\partial z(x_0,y_0,w_0^*)^T}{\partial (w_0^*)^T}\nabla^2L_{baen}(z(x_0,y,w_0^*))\,dF\right)
\)is invertible.

Assumption 3 is quite common in statistics and is easily satisfied when the sample dimension is finite. If $W_0$ is not invertible, then one eigenvalue of
\(\int xx^T\nabla^2L_{baen}(z(x,y,w_0^*))\,dF\)
is exactly equal to $\frac{1}{nC}$, which is a small probability event. Therefore, Assumptions 3 and 4 are easy to satisfy.

\begin{theorem}
For the linear $\varepsilon$-BAEN-SVM with given $\lambda$, $\tau$, and $p$, the influence function at the sample point $(x_0^T,y_0)^T$ is
\begin{equation}
\label{eq:327}
\mathrm{IF}(x_0,y_0;w_0^*)
=
W_0^{-1}
\left(
-\frac{1}{nC}w_0
-\gamma_0
-\nabla L_{baen}^{\varepsilon}(z(x_0,y_0,w_0^*))\,
\frac{\partial z(x_0,y_0,w_0^*)}{\partial w_0^*}
\right),
\end{equation}
where
\(
W_0=
\left(
\frac{1}{nC}I+\int xx^T\nabla^2L_{baen}^{\varepsilon}(z(x_0,y,w_0^*))\,dF
\right),
z=1-yx^Tw_{\theta}^{*},
\)

\[
\gamma_1=\left.\int \frac{\partial}{\partial\theta}
\left(
\zeta_1\cdot\frac{\partial z(x,y,w_{\theta}^{*})}{\partial w_{\theta}^{*}}
\right)dF\right|_{\theta=0},
\gamma_2=\left.\int \frac{\partial}{\partial\theta}
\left(
\zeta_2\cdot\frac{\partial z(x,y,w_{\theta}^{*})}{\partial w_{\theta}^{*}}
\right)dF\right|_{\theta=0}.
\]
where $\zeta_1(\theta,x,y)\in\left[0,\dfrac{\eta(1-p)}{\lambda}\right]$ and $\zeta_2(\theta,x,y)\in\left[-\dfrac{\eta\tau(1-p)}{\lambda},0\right]$, and the influence function $\mathrm{IF}(x_0,y_0;w_0^*)$ is bounded.
\end{theorem}\label{th1}

\begin{proof}
According to the KKT conditions, $w_{\theta}^{*}$ satisfies
\begin{equation}
\label{eq:328}
w_{\theta}^{*}
=
-nC\int
\left[
\nabla L_{baen}^{\varepsilon}(z(x,y,w_{\theta}^{*}))
\frac{\partial z(x,y,w_{\theta}^{*})}{\partial w_{\theta}^{*}}
\right]dF_{\theta}.
\end{equation}
where
\begin{equation}
\label{eq:329}
\nabla L_{baen}^{\varepsilon}(z)=
\begin{cases}
\dfrac{\eta\bigl(p(z-\varepsilon)+(1-p)\bigr)}
{\lambda\left(1+\eta\left(\dfrac{p}{2}(z-\varepsilon)^2+(1-p)(z-\varepsilon)\right)\right)^2},
& z>\varepsilon,\\[10pt]
\left[0,\dfrac{\eta(1-p)}{\lambda}\right],
& z=\varepsilon,\\[10pt]
0,
& -\dfrac{\varepsilon}{\tau}<z<\varepsilon,\\[10pt]
\left[-\dfrac{\eta\tau(1-p)}{\lambda},0\right],
& z=-\dfrac{\varepsilon}{\tau},\\[10pt]
\dfrac{\eta\bigl(p(z+\varepsilon/\tau)-(1-p)\bigr)}
{\lambda\left(1+\eta\left(\dfrac{p}{2}(z+\varepsilon/\tau)^2-(1-p)(z+\varepsilon/\tau)\right)\right)^2},
& z<-\dfrac{\varepsilon}{\tau}.
\end{cases}
\end{equation}
Substituting $F_{\theta}=(1-\theta)F+\theta p_0$ into (\ref{eq:328}), we obtain
\begin{equation}
\label{eq:330}
-\frac{1}{nC}w_{\theta}^{*}
=
(1-\theta)\int \nabla L_{baen}^{\varepsilon}(z(x,y,w_{\theta}^{*}))
\frac{\partial z(x,y,w_{\theta}^{*})}{\partial w_{\theta}^{*}}\,dF
+\theta \nabla L_{baen}^{\varepsilon}(z(x_0,y_0,w_{\theta}^{*}))
\frac{\partial z(x_0,y_0,w_{\theta}^{*})}{\partial w_{\theta}^{*}}.
\end{equation}
Taking derivatives of both sides of (\ref{eq:330}) with respect to $\theta$, and letting $\theta\to 0$, yields
\begin{equation}
\label{eq:331}
\begin{aligned}
\left.\frac{1}{nC}\frac{\partial w_{\theta}^{*}}{\partial \theta}\right|_{\theta=0}
={}&
\left.\int \nabla L_{baen}^{\varepsilon}(z(x,y,w_{\theta}^{*}))
\frac{\partial z(x,y,w_{\theta}^{*})}{\partial w_{\theta}^{*}}\,dF\right|_{\theta=0} \\
&-
\left.\int \nabla^{2}L_{baen}^{\varepsilon}(z(x,y,w_{\theta}^{*}))
\frac{\partial z(x,y,w_{\theta}^{*})}{\partial w_{\theta}^{*}}
\frac{\partial z(x,y,w_{\theta}^{*})}{\partial (w_{\theta}^{*})^{T}}\,dF
\frac{\partial w_{\theta}^{*}}{\partial \theta}\right|_{\theta=0} \\
&-
\left.\nabla L_{baen}^{\varepsilon}(z(x_0,y_0,w_{\theta}^{*}))
\frac{\partial z(x_0,y_0,w_{\theta}^{*})}{\partial w_{\theta}^{*}}\right|_{\theta=0}
-\gamma_1-\gamma_2 .
\end{aligned}
\end{equation}

Here,
\[
\left.\gamma_1=\int \frac{\partial}{\partial \theta}
\left(
\zeta_1\cdot \frac{\partial z(x,y,w_{\theta}^{*})}{\partial w_{\theta}^{*}}
\right)dF\right|_{\theta=0},
\zeta_1(\theta,x,y)\in \left[0,\frac{\eta(1-p)}{\lambda}\right],
\]
\[
\left.\gamma_2=\int \frac{\partial}{\partial \theta}
\left(
\zeta_2\cdot \frac{\partial z(x,y,w_{\theta}^{*})}{\partial w_{\theta}^{*}}
\right)dF\right|_{\theta=0},
\zeta_2(\theta,x,y)\in \left[-\frac{\eta\tau(1-p)}{\lambda},0\right],
\]
where $\gamma_1$ and $\gamma_2$ come from the first order derivative of the loss function $L_{baen}^{\varepsilon}$.

Combining (\ref{eq:328}) and (\ref{eq:331}), we obtain
\begin{equation}
\label{eq:332}
\begin{aligned}
\left(
\frac{1}{nC}I
+\int \nabla^{2}L_{baen}^{\varepsilon}(z(x_0,y_0,w_0^{*}))
\frac{\partial z(x_0,y_0,w_0^{*})}{\partial w_0^{*}}
\frac{\partial z(x_0,y_0,w_0^{*})}{\partial (w_0^{*})^{T}}\,dF
\right)
\mathrm{IF}(x_0,y_0;w_0^{*}) \\
=
-\frac{1}{nC}w_0
-\nabla L_{baen}^{\varepsilon}(z(x_0,y_0,w_0^{*}))
\frac{\partial z(x_0,y_0,w_0^{*})}{\partial w_0^{*}}
-\gamma_1-\gamma_2 .
\end{aligned}
\end{equation}

Here, $I$ is the identity matrix, and let
\[
W_0=
\left(
\frac{1}{nC}I
+\int
\frac{\partial z(x_0,y_0,w_0^{*})}{\partial w_0^{*}}
\frac{\partial z(x_0,y_0,w_0^{*})}{\partial (w_0^{*})^{T}}
\nabla^{2}L_{baen}^{\varepsilon}(z(x_0,y_0,w_0^{*}))\,dF
\right).
\]
According to Assumption 3, the influence function can be derived as
\begin{equation}
\label{eq:333}
\mathrm{IF}(x_0,y_0;w_0^{*})
=
W_0^{-1}
\left(
-\frac{1}{nC}w_0
-\gamma_1
-\gamma_2
-\nabla L_{baen}^{\varepsilon}(z(x_0,y_0,w_0^{*}))
\frac{\partial z(x_0,y_0,w_0^{*})}{\partial w_0^{*}}
\right).
\end{equation}

Next, we prove that the influence function is bounded. By Assumption 3 and \ref{eq:333}, we obtain the following inequality:
\begin{equation}
\label{eq:334}
\|\mathrm{IF}(x_0,y_0;w_0^{*})\|
\leq
\lambda_{\min}(W_0)
\left(
\left\|\frac{1}{nC}w_0\right\|
+\|\gamma_1\|
+\|\gamma_2\|
+\left\|
\frac{\partial z(x_0,y_0,w_0^{*})}{\partial w_0^{*}}
\right\|
\left\|
\nabla L_{baen}^{\varepsilon}(z(x_0,y_0,w_0^{*}))
\right\|
\right).
\end{equation}

Here, $\lambda_{\min}(W_0)$ denotes the minimum eigenvalue of the matrix $W_0$. Since $\zeta_1$ and $\zeta_2$ are bounded and continuous with respect to $\theta$ over the interval, their corresponding derivatives $\gamma_1$ and $\gamma_2$ are also bounded with respect to $\theta$. Moreover,
\(
\left\|
\frac{\partial z(x_0,y_0,w_0^{*})}{\partial w_0^{*}}
\right\|
=x_0.
\)
When $\|x_0\|<\infty$,
\(
\left\|
\nabla L_{baen}^{\varepsilon}(z(x_0,y_0,w_0^{*}))
\right\|
\)
is bounded; when $\|x_0\|\to\infty$,
\(
\left\|
\nabla L_{baen}^{\varepsilon}(z)
\right\|
\to 0.
\)
Therefore,
\[
\lambda_{\min}(W_0)
\left(
\left\|\frac{1}{nC}w_0\right\|
+\|\gamma_1\|
+\|\gamma_2\|
+\left\|
\frac{\partial z(x_0,y_0,w_0^{*})}{\partial w_0^{*}}
\right\|
\left\|
\nabla L_{baen}^{\varepsilon}(z(x_0,y_0,w_0^{*}))
\right\|
\right)
\leq \infty.
\]

In summary, the influence function of $\varepsilon$-BAEN-SVM is bounded. Therefore, $\varepsilon$-BAEN-SVM is robust to label noise. The proof of Theorem \ref{th1} is complete.
\end{proof}

	    \subsubsection{Robust to Feature noise}
	
The previous subsection has shown that minimizing the risk of the loss $L_{baen}^{\varepsilon}$ leads to a Bayes classifier that is robust to label noise. In this subsection, the method proposed by \citep{huang2013support} is adopted to prove the robustness of $\varepsilon$-BAEN-SVM to feature noise.

By the KKT conditions, the optimization problem satisfied by the solution of $\varepsilon$-BAEN-SVM  can be expressed as
\begin{equation}
\label{eq:335}
0 \in \frac{w}{C}-\frac{1}{2}\nabla L_{baen}^{\varepsilon}(1-y_i x_i^T w; p,\tau,\lambda,\eta)\,y_i x_i .
\end{equation}

Here, $0$ denotes a zero vector of appropriate dimension whose entries are all zero.

According to the subgradient of $L_{baen}^{\varepsilon}$ in \ref{eq:329}, for a given $w$, the training samples can be divided into the following five classes:
\begin{equation}
\label{eq:336}
\begin{cases}
S_1^w=\{i:1-y_i x_i^T w>\varepsilon\},\\[4pt]
S_2^w=\{i:1-y_i x_i^T w=\varepsilon\},\\[4pt]
S_3^w=\left\{i:-\dfrac{\varepsilon}{\tau}<1-y_i x_i^T w<\varepsilon\right\},\\[8pt]
S_4^w=\left\{i:1-y_i x_i^T w=-\dfrac{\varepsilon}{\tau}\right\},\\[8pt]
S_5^w=\left\{i:1-y_i x_i^T w<-\dfrac{\varepsilon}{\tau}\right\}.
\end{cases}
\end{equation}

Since there exist $\zeta_1(\theta,x,y)\in\left[0,\dfrac{\eta(1-p)}{\lambda}\right]$ and $\zeta_2(\theta,x,y)\in\left[-\dfrac{\eta\tau(1-p)}{\lambda},0\right]$, the optimality condition in \ref{eq:335} can be written as
\begin{align}
 &\frac{w}{C}
-\sum_{i\in S_1^w}
\frac{\eta\tau\bigl(p(z-\varepsilon)+(1-p)\bigr)}
{\lambda\left(1+\eta\left(\dfrac{p}{2}(z-\varepsilon)^2+(1-p)(z-\varepsilon)\right)\right)^2}y_i x_i
-\sum_{i\in S_2^w}\zeta_{1i}y_i x_i
\\
&-\sum_{i\in S_4^w}\zeta_{2i}y_i x_i\nonumber-\sum_{i\in S_5^w}
\frac{\eta\bigl(p(z+\varepsilon/\tau)-(1-p)\bigr)}
{\lambda\left(1+\eta\left(\dfrac{p}{2}(z+\varepsilon/\tau)^2-(1-p)(z+\varepsilon/\tau)\right)\right)^2}y_i x_i
=0.
\end{align}\label{eq:337}

Because the sets $S_2^w$, $S_3^w$, and $S_4^w$ are determined by equalities, it is reasonable to infer that the cardinalities of $S_2^w$, $S_3^w$, and $S_4^w$ are much smaller than those of $S_1^w$ and $S_5^w$. Therefore, the contributions of $S_2^w$, $S_3^w$, and $S_4^w$ to \ref{eq:337} are relatively weak. Hence, $w$ can be approximately determined according to $S_1^w$ and $S_5^w$. Thus, \ref{eq:337} becomes
\begin{align}
\label{eq:338}
&\frac{w}{C}
-\sum_{i\in S_1^w}
\frac{\eta\tau\bigl(p(z-\varepsilon)+(1-p)\bigr)}
{\lambda\left(1+\eta\left(\dfrac{p}{2}(z-\varepsilon)^2+(1-p)(z-\varepsilon)\right)\right)^2}y_i x_i\nonumber\\
&-\sum_{i\in S_5^w}
\frac{\eta\bigl(p(z+\varepsilon/\tau)-(1-p)\bigr)}
{\lambda\left(1+\eta\left(\dfrac{p}{2}(z+\varepsilon/\tau)^2-(1-p)(z+\varepsilon/\tau)\right)\right)^2}y_i x_i
\approx 0.
\end{align}

Since $\lambda>0$ and $\eta>0$, \ref{eq:338} can be rewritten as
\begin{equation}
\label{eq:339}
\frac{w}{C}
+\sum_{i\in S_1^w}\tau\bigl(1-p-p(1-y_i x_i^T w-\varepsilon)\bigr)y_i x_i
-\sum_{i\in S_5^w}\bigl(p(1-y_i x_i^T w+\varepsilon/\tau)+1-p\bigr)y_i x_i
\approx 0.
\end{equation}

By properly choosing the parameters $\varepsilon$ and $\tau$, the sensitivity of the model to feature noise can be controlled. Specifically, when $\tau$ is large, that is, close to $1$, both $S_1^w$ and $S_5^w$ contain a large number of sample points. In this case, the model achieves a better balance between $S_1^w$ and $S_5^w$, and the contributions of samples on both sides of the decision boundary constrain each other, which helps reduce the sensitivity to zero mean feature noise. On the other hand, decreasing $\varepsilon$ also increases the number of samples in $S_1^w$ and $S_5^w$, thereby making the model less affected by zero mean feature noise near the decision boundary. Therefore, it can be concluded that $\varepsilon$-BAEN-SVM is robust to feature noise. In addition, increasing the value of $\varepsilon$ increases the samples corresponding to $S_3^w$, which makes the model sparser. To a certain extent, this indicates that adjusting the size of $\varepsilon$ helps balance the sparsity and robustness of the model.        
         
        \subsection{Complexity Analysis}
        This subsection provides a detailed analysis of the time complexity of the proposed BAEN-SVM method. Our algorithm has a computational advantage over existing algorithms designed for solving non-convex models, primarily owing to its efficient strategy for addressing the associated quadratic optimization subproblem.  

        Specifically, each iteration of \autoref{alg:hqclipdcd} need to solve a quadratic programming (QP) problem. In general, the time complexity of solving such a QP problem is $O((2n)^3)$, where $n$ denotes the number of training samples. However, by employing the clipDCD algorithm \citep{boyd2004convex}, we can reduce the complexity of each coordinate update to $O(2n)$. The clipDCD algorithm's overall time complexity is $O(t(2n))$ if convergence occurs after $t$ iterations. Therefore, we adopt the clipDCD algorithm for the BAEN-SVM subproblem. Let $q$ denote the number of iterations required for convergence for the half-quadratic optimization procedure. Then, the overall time complexity for computing \autoref{alg:hqclipdcd} is $O(qt(2n))$, where $q$ and $t$ refer to the number of HQ and clipDCD iterations, respectively. Consequently, compared to the direct solution method with complexity $O(q(2n)^3)$, implementing the clipDCD-based HQ optimization method significantly reduces computational complexity, especially for large-scale datasets.

    \section{Experiments}
        \subsection{Set up}
        In this section, we present several experiments to evaluate the performance of the proposed $\epsilon$-BAEN-SVM on both artificial and benchmark datasets. For fair assessment and comprehensive comparison, the comparison models include well-known or recently proposed SVMs, such as Pin-SVM \citep{PinSVM6604389}, ALS-SVM \citep{HUANGALS}, EN-SVM \citep{qi2019ENSVM}, BQ-SVM \citep{ZHANG2024BQSVM}, BALS-SVM \citep{zhang2025BALSSVM}, and BAEN-SVM. 
        The algorithms are implemented in R 4.4.2, and the experiments are conducted on a machine equipped with an AMD Ryzen 7 8845H CPU (3.80 GHz) and 32GB of RAM.
            
        Five-fold cross-validation and grid search methods are applied to select the optimal settings for each model.
        The parameter $C_1$ and $C_2$  in EN-SVM have a range of values between $\{2^{-8},2^{-6}, 2^{-4}, \cdots, 2^{4}, 2^{6},2^{8}\}$.
        The parameters $p$ in $\epsilon$-BAEN-SVM, ALS-SVM ,BALS-SVM and BAEN-SVM are selected from  $\{0.3,0.5, 0.7\}$, $\{0.5, 0.7, 0.9, 0.99, 0.999\}$ ,$\{0.3,0.5, 0.7, 0.9, 0.99\}$ and $\{0.3,0.5, 0.7\}$,respectively. Set the parameter $\epsilon$ of the $\epsilon$-BAEN-SVM model to 0.1.
        The parameters $\tau$ in $\epsilon$-BAEN-SVM, BAEN-SVM, BQ-SVM and Pin-SVM are selected from $\{0.1, 0.3, 0.6, 1\}$, $\{0, 0.1, 0.3, 0.6, 1\}$, $\{0, 0.1, 0.3, 0.6, 1\}$ and $\{0.1, 0.3, 0.6, 1\}$, respectively. 
        The parameters $\eta$ of BALS-SVM, BQ-SVM, BAEN-SVM, and $\epsilon$-BAEN-SVM takes on values in $\{2^{-6}, 2^{-4}, \cdots, 2^{4}, 2^{6}\}$.  For grid-searching the SVM regularization parameter $C$, we have $C\in \{2^{i}\}$, where $i\in\{-8,-7,\cdots,8\}$.
        For the nonlinear case, we use a radial basis function (RBF) kernel\begin{equation}K(x_i,x_j)=\exp(-\sigma\|x_i-x_j\|_2^2),\end{equation}with \(\sigma\) chosen from \(\{2^{-4},2^{-3},\cdots,2^3,2^4\}\). 

        The accuracy (ACC) and \(F_1\mathrm{-score~}(F_1)\) are used to evaluate the classification performance of BAEN-SVM.
        Accuracy measures the proportion of samples correctly predicted by the model out of the total samples, which is defined as
             \begin{equation}
             ACC=\frac{TP+TN}{TP+TN+FP+FN},
             \end{equation}
         where $TP$ and $TN$ represent the number of correctly predicted positive and negative samples, respectively, while $FP$ and $FN$ reflect the number of misclassified positive and negative samples.
         
        The $F_1$ score is the reconciled average of precision and recall, which is expressed as
             \begin{equation}
              F_1 = \frac{2 \cdot \text{Precision} \cdot \text{Recall}}{\text{Precision} + \text{Recall}} = \frac{2TP}{2TP + FP + FN}.
            \end{equation}
        
        Precision measures how well a model avoids labeling negative samples as positive. A higher precision means fewer negative samples are misclassified. Recall measures how well a model finds positive samples. A higher recall means fewer positive samples are missed. A larger $F_1$ value signifies greater model robustness. Both $ACC$ and $F_1$ values range from 0 to 1, with higher values indicating superior model performance.
              
        \subsection{Artificial Datasets}
        We create a two-dimensional artificial dataset of 150 samples equally divided between two classes. Positive and negative samples are drawn from normal distributions with \(\mu_+=(3,3)^T\) and \(\mu_{-}=(-3,-3)^{T}\), respectively, and share the covariance matrix \(V=\operatorname{di}\operatorname{ag}(1,1)\). For this experiment, the Bayes classifier is given by \(f_{C}(x)=x_{1}-x_{2}\).
            
        Case 1. We introduce three outliers (label noise) into the negative class to simulate data contamination. \autoref{fig:moni1} illustrates a comparison of the classification boundaries (black solid line) derived from six SVMs with the Bayes optimum boundary (green solid line). The deviation of each model's decision boundary from the Bayes classifier reflects its sensitivity to the introduced label noise.
        
            \begin{figure}[H]
		\centering
		\subfigure[Hinge-SVM]{
			\includegraphics[scale=0.35]{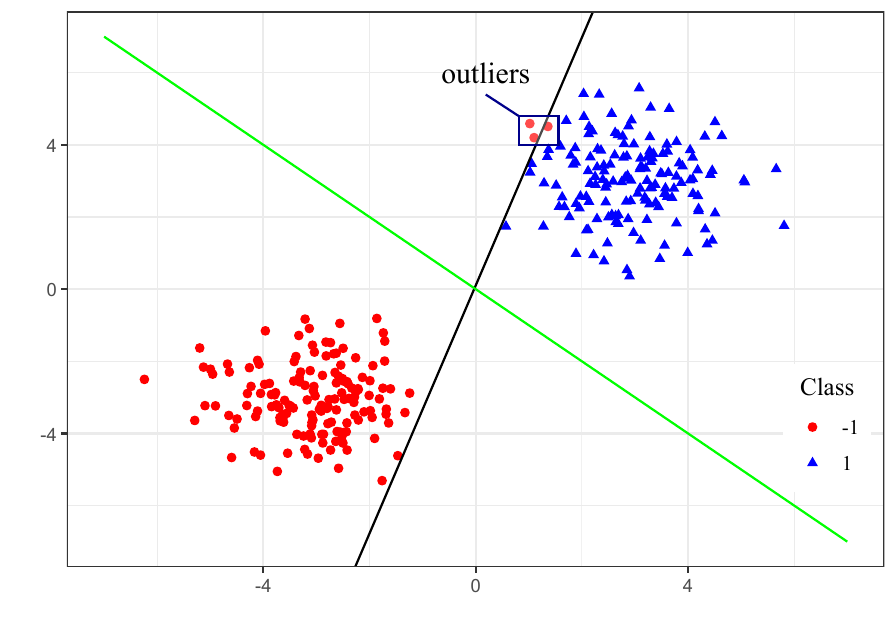}
		}
		\subfigure[Pin-SVM]{
			\includegraphics[scale=0.35]{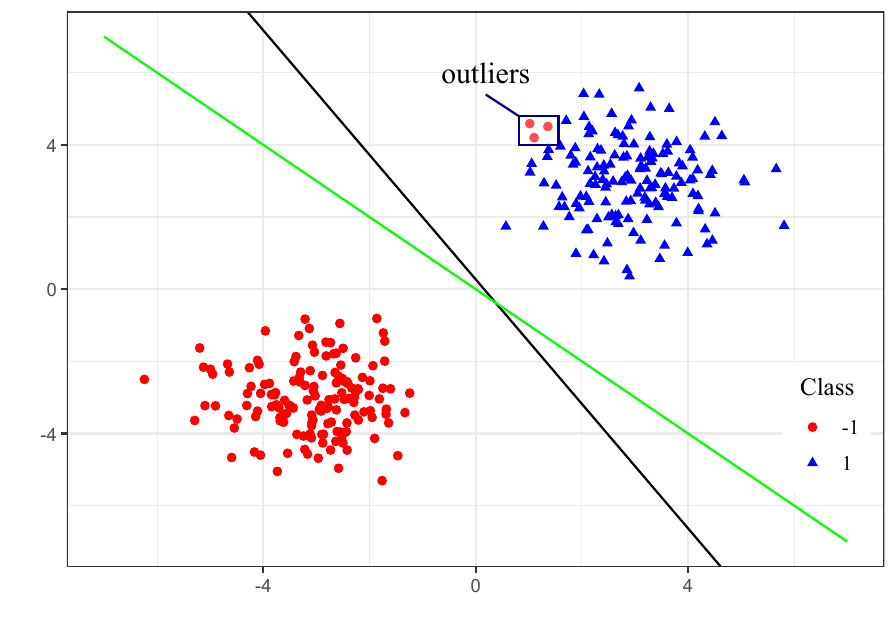}
		}
		\subfigure[LS-SVM]{
			\includegraphics[scale=0.35]{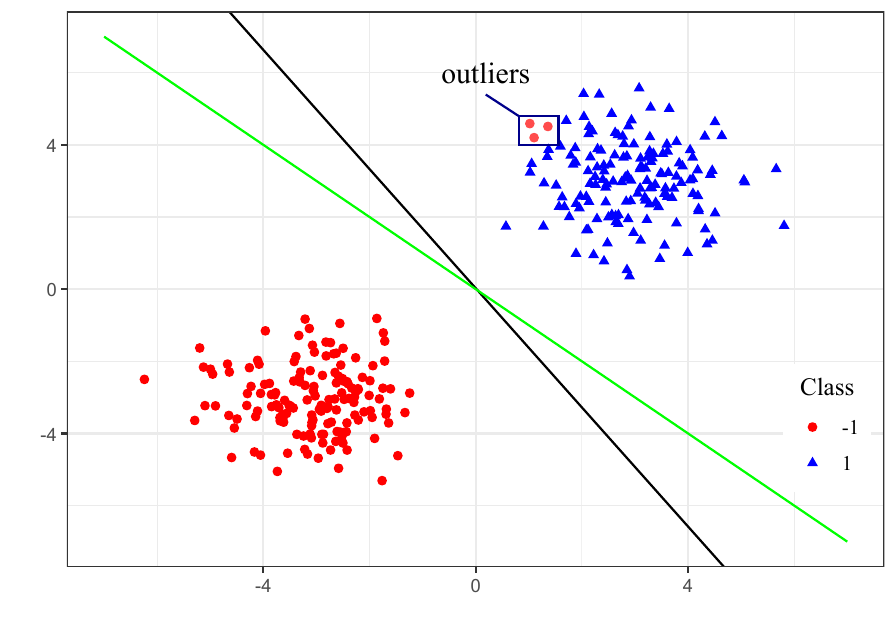}
		}
            \\
		\centering
		\subfigure[ALS-SVM]{
			\includegraphics[scale=0.35]{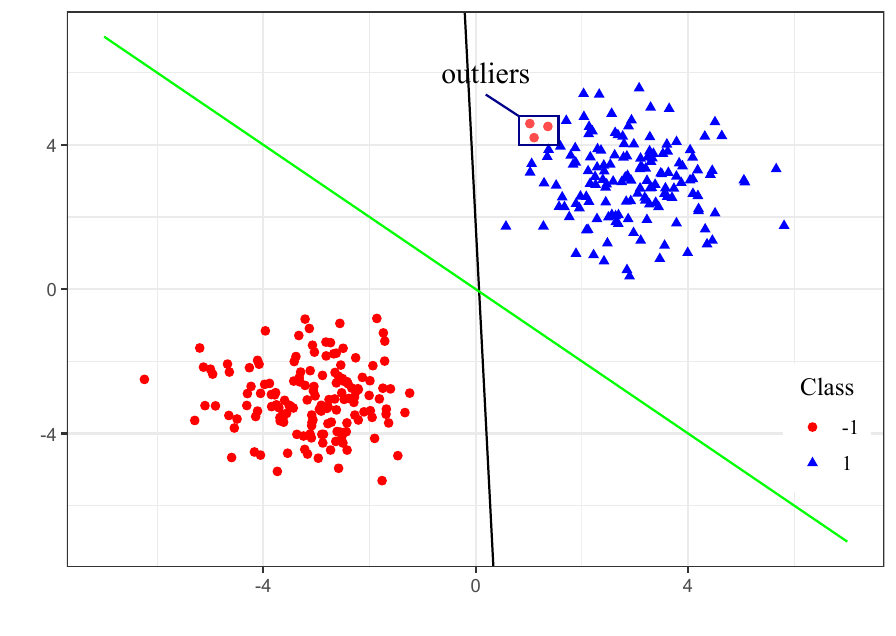}
		}
		\subfigure[EN-SVM]{
			\includegraphics[scale=0.35]{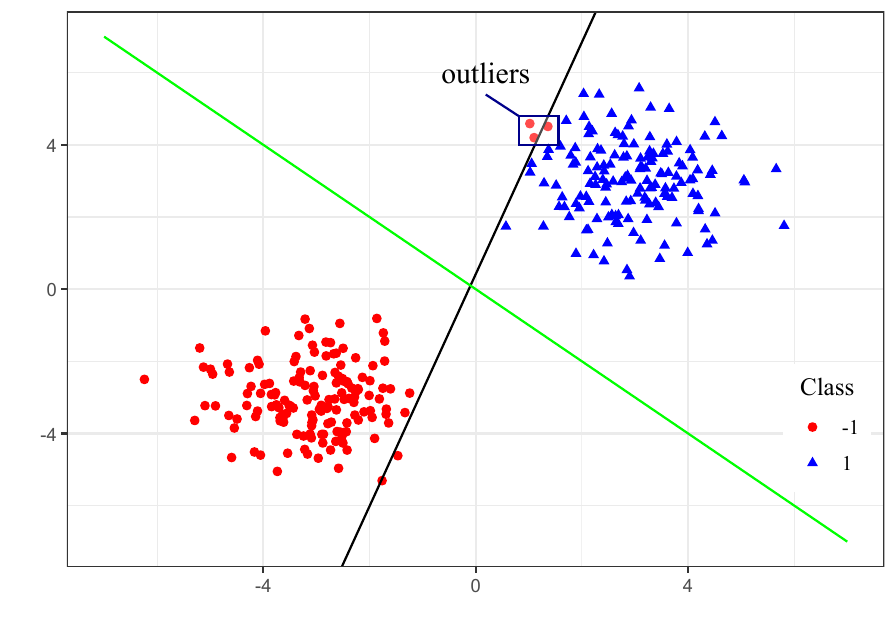}
		}
		\subfigure[$\epsilon$-BAEN-SVM]{
			\includegraphics[scale=0.35]{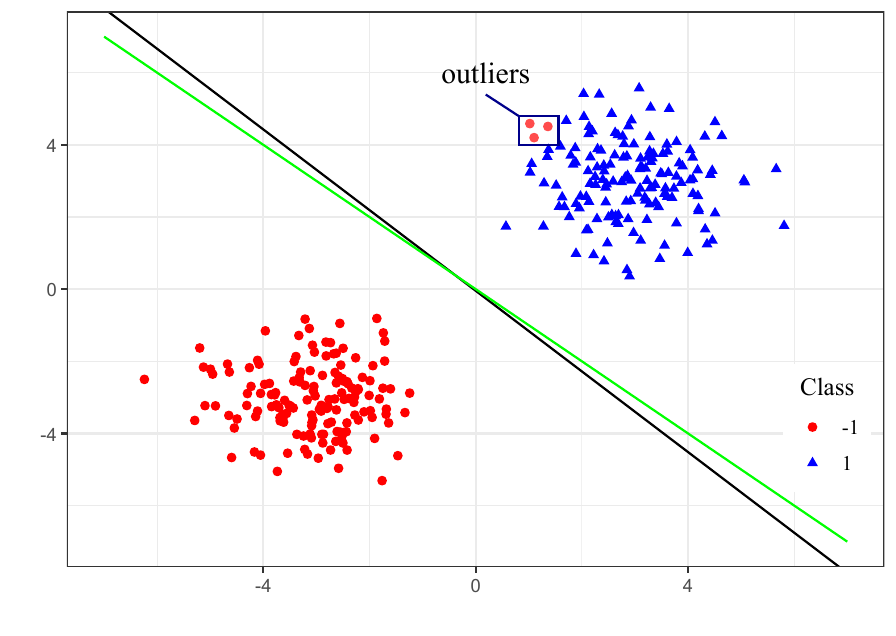}
		}
		\caption{Linear separating hyperplanes(black solid lines of Hing-SVM,Pin-SVM,LS-SVM,ALS-SVM,EN-SVM,BAEN-SVM. The green solid line is the Bayes classifier.}
		\label{fig:moni1}
	\end{figure}

        In \autoref{fig:moni1}, $\epsilon$-BAEN-SVM exhibits the most stable performance in the presence of outliers, closely aligning with the Bayes optimal boundary and outperforming the other methods. LS-SVM and Pin-SVM follow, with their classification decisions slightly deviating from the Bayes classifier due to label noise. In contrast, Hinge-SVM and EN-SVM perform poorly, as their decision boundaries significantly deviate from the Bayes classifier, highlighting their high sensitivity to label noise.
            
        Case 2. In this case, three outliers are introduced into both the positive and negative classes. \autoref{fig:moni2} displays the training samples along with the decision boundaries (black solid lines) generated by six different SVM models. The green solid line is the Bayes classifier.
              \begin{figure}[htbp]
		\centering
		\subfigure[Hinge-SVM]{
			\includegraphics[scale=0.35]{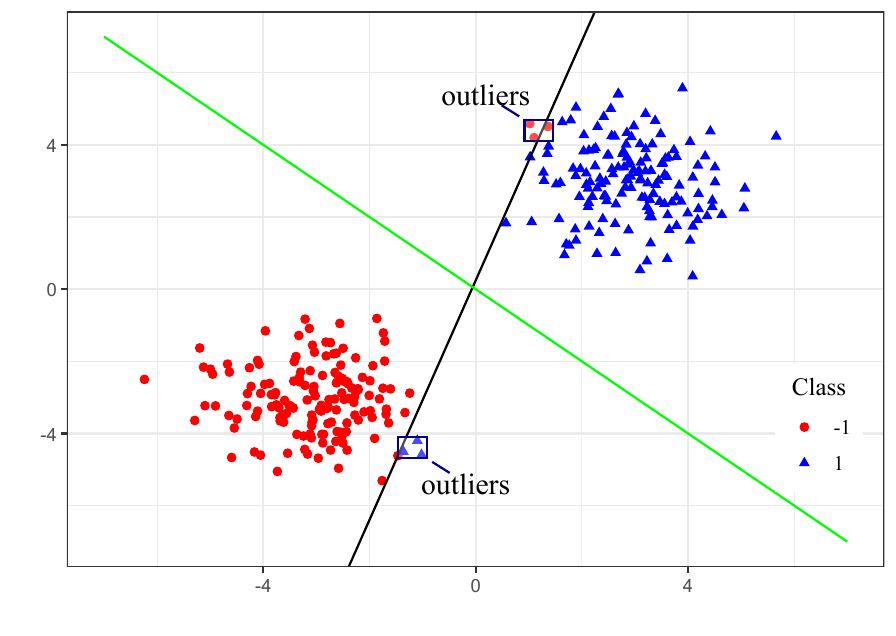}
		}
		\subfigure[Pin-SVM]{
			\includegraphics[scale=0.35]{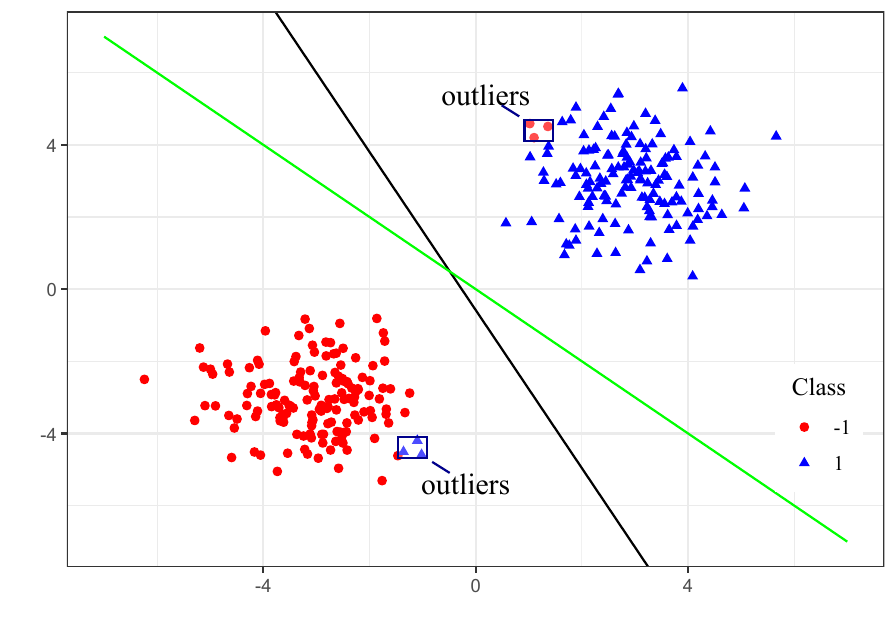}
		}
		\subfigure[LS-SVM]{
			\includegraphics[scale=0.35]{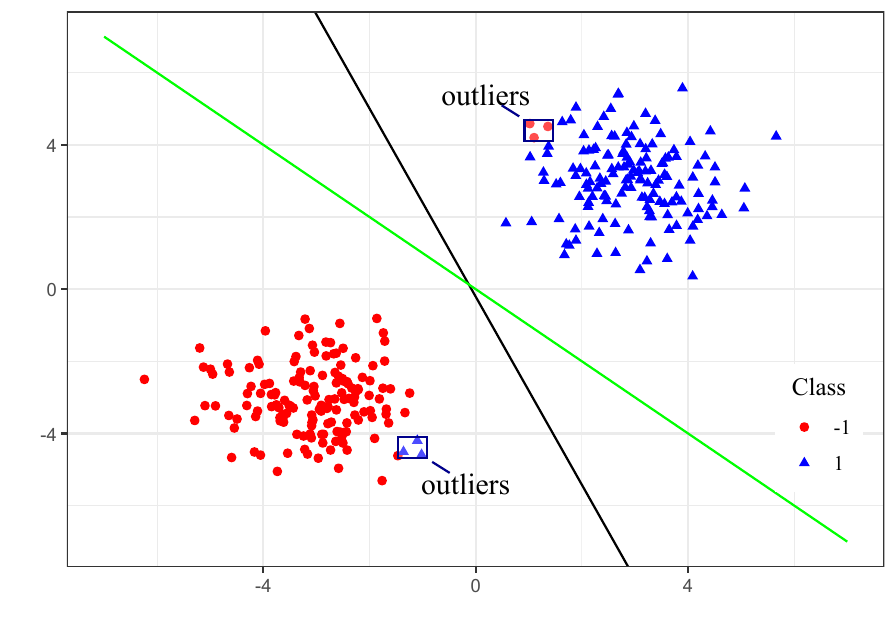}
		}
            \\
		\centering
		\subfigure[ALS-SVM]{
			\includegraphics[scale=0.35]{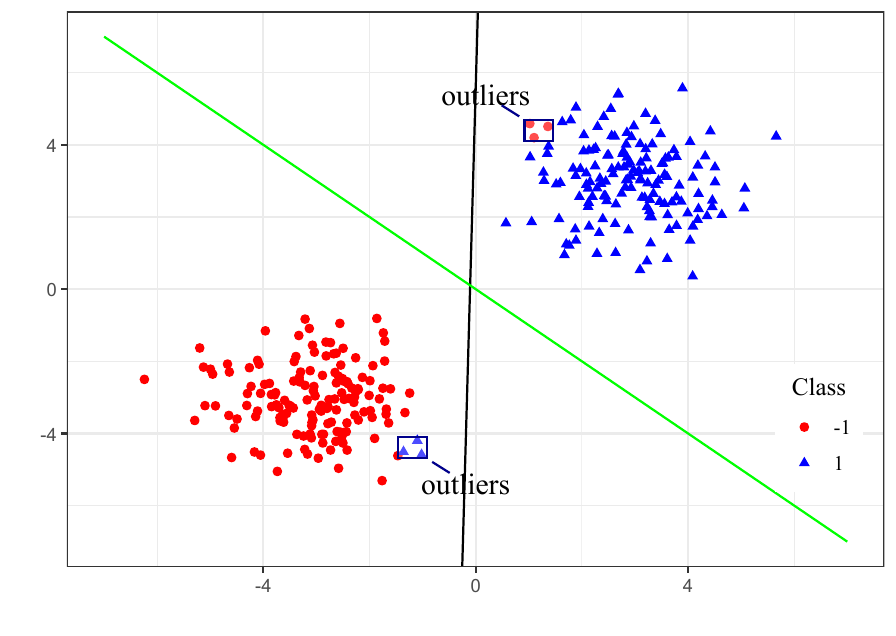}
		}
		\subfigure[EN-SVM]{
			\includegraphics[scale=0.35]{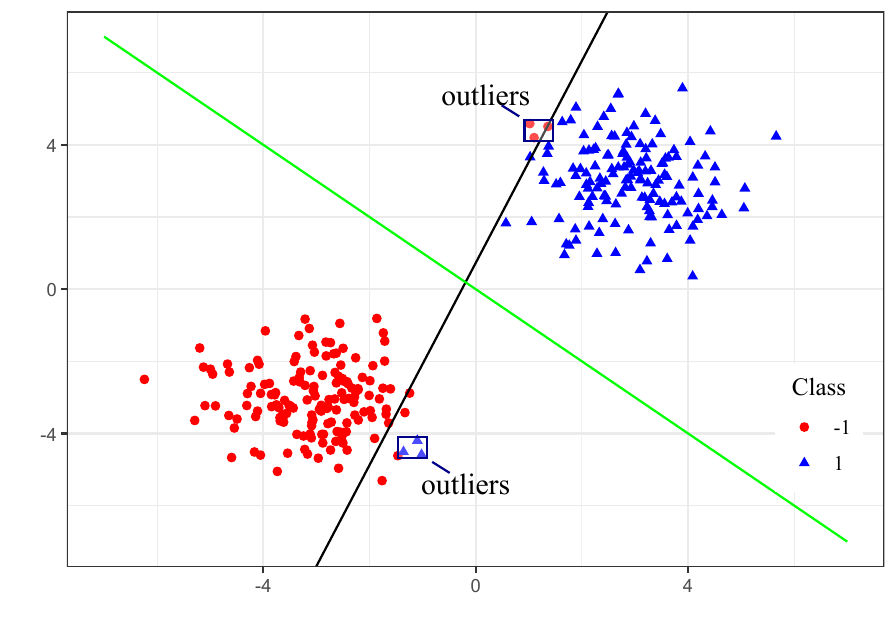}
		}
		\subfigure[$\epsilon$-BAEN-SVM]{
			\includegraphics[scale=0.35]{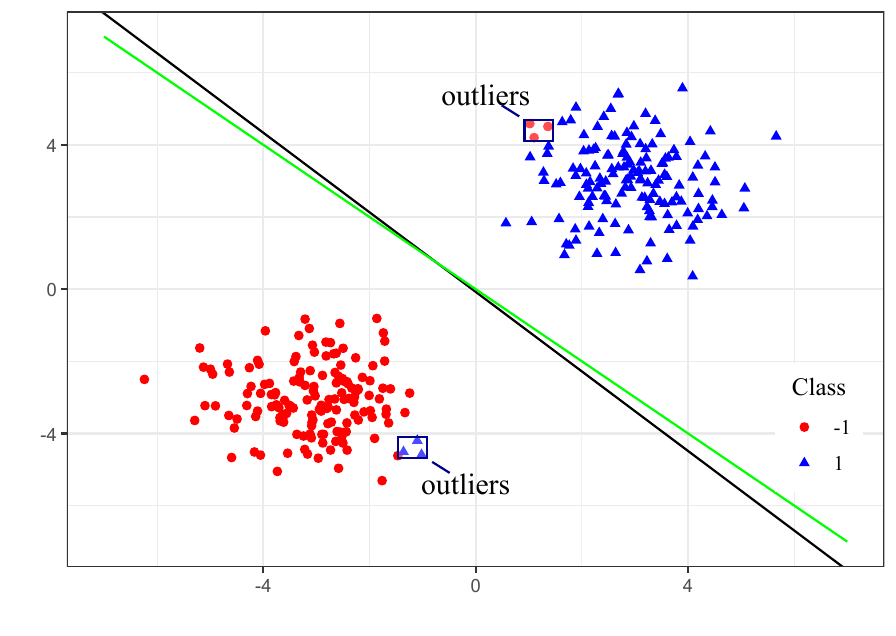}
		}
		\caption{Linear separating hyperplanes(black solid lines of Hing-SVM,Pin-SVM,LS-SVM,ALS-SVM,EN-SVM,$\epsilon$-BAEN-SVM. The green solid line is the Bayes classifier.}
		\label{fig:moni2}
	\end{figure}

        As shown in \autoref{fig:moni2}, $\epsilon$-BAEN-SVM maintains superior classification performance even when outliers are added to both classes. In contrast, EN-SVM and Hinge-SVM are significantly affected by the outliers. Their decision boundaries deviate significantly and even intersect the outlier points, which indicates they appear to be overfitted. While Pin-SVM and LS-SVM exhibit some deviation from the Bayes optimal boundary, their performance still outperforms that of ALS-SVM, Hinge-SVM, and EN-SVM. Overall, $\epsilon$-BAEN-SVM exhibits the strongest robustness among all models, which aligns with its boundness. This result is consistent with the theoretical conclusion in \autoref{th: if baen}, which further validates that BAEN-SVM is highly robust to label noise.
        
        \subsection{Benchmark Datasets}
        We select 15 datasets from the UCI machine learning repository\footnote{\url{https://archive.ics.uci.edu/}} and the homepage of KEEL\footnote{\url{https://sci2s.ugr.es/keel/datasets.php}} to further validate the competitive performance of BAEN-SVM. Detailed descriptions of datasets are provided in \autoref{tab: data des}.
        \begin{table}[htbp]
		\centering
		\footnotesize
		\caption{Description of fifteen benchmark datasets}
		\begin{tabularx}{\hsize}{@{}@{\extracolsep{\fill}}ZZZZ@{}}
			\toprule
			ID & Dataset & Samples & Attributes  \\
			\midrule
			1  & appendicitis   & 106 & 7    \\
			2  & blood          & 748& 4     \\
            3  & coimbra        & 116&9      \\ 
			4  & diabetic       & 1151 & 19  \\
			5  & fertility      & 100 & 9    \\
			6  & haberman       & 306 & 3    \\
            7  &heart failure   & 299 & 12   \\
            8  &monkm           &431  &6     \\
			9 & pima           & 768  &8    \\
			10 & plrx           & 182& 12    \\
			11 & pop failures   & 540 & 20   \\
			12 & sonar          & 208 & 60   \\
            13 &titanic         &2200 &3     \\
            14 &knowledge       &403  &5     \\
            15 &wisconsin       &699  &9     \\
			\bottomrule
	   \end{tabularx}
		 \label{tab: data des}
	\end{table}
        To further assess the robustness to noise, we artificially add 25\% label noise by randomly swapping 25\% labels in all samples. Additionally, feature noise is added by generating zero-mean Gaussian noise for each feature, with the noise variance scaled by the feature's original variance. The noise level is controlled by the ratio $r$, which represents the proportion of the noise variance relative to the feature variance. The results of BAEN-SVM and the baseline models with linear kernel based on five-fold cross-validation are shown in \autoref{tab:res classification linear acc} and \autoref{tab:res classification linear F1}. The results for Gaussian kernel are shown in \autoref{tab:res classification Gaussian acc} and \autoref{tab:res classification Gaussian F1}. 
            
         From \autoref{tab:res classification Gaussian acc}  and \autoref{tab:res classification Gaussian F1}, in linear conditions, the proposed $\varepsilon$-BAEN-SVM achieves higher average accuracy and a higher $F_1$ score than other methods. This advantage becomes even clearer when we add 25\% feature noise or label noise. This result further confirms the noise robustness of $\varepsilon$-BAEN-SVM. In noisy datasets, BAEN-SVM and BQ-SVM perform nearly as well as $\varepsilon$-BAEN-SVM. EN-SVM always outperforms Pin-SVM and ALS-SVM under both no-noise and 25\% feature noise conditions. This shows the strength of the elastic network loss function. However, adding 25\% label noise greatly reduces the performance of EN-SVM. The reason is that its loss function is not robust. For example, on the diabetic dataset, the accuracy of EN-SVM drops from 0.735 to 0.665. In contrast, the designed loss function $L_{\text{baen}}^{\varepsilon}$ has boundedness and asymmetry. These properties help $\varepsilon$-BAEN-SVM remain stable against outliers and resampling. As a result, $\varepsilon$-BAEN-SVM maintains high average accuracy and high $F_1$ scores under both label noise and feature noise. This confirms that the model is effective and robust on linearly separable data.

 \begin{table}[htbp]
		\centering
		\fontsize{10pt}{5pt}\selectfont
		\caption{Comparison of the mean accuracy (ACC$\pm$sd) of seven SVMs with linear kernel in benchmark datasets.}
		\label{tab:res classification linear acc}
		\begin{tabularx}{\hsize}{@{}@{\extracolsep{\fill}}ZZZZZZZZ@{}}
			\toprule 
			\multicolumn{8}{c}{(a) 0\% noise} \\ 
			\midrule 
			dataset& Pin-SVM & ALS-SVM & EN-SVM & BQ-SVM & BALS-SVM & BAEN-SVM & $\epsilon$-BAEN-SVM  \\
			\midrule 
australian&0.857$\pm$0.039&0.878$\pm$0.017&0.868$\pm$0.021&0.875$\pm$0.013&\textbf{0.881$\pm$0.015}&\textbf{0.881$\pm$0.015}&0.868$\pm$0.021\\
blood&0.771$\pm$0.027&0.777$\pm$0.028&0.777$\pm$0.030&0.774$\pm$0.032&0.777$\pm$0.028&0.778$\pm$0.027&\textbf{0.789$\pm$0.039}\\
coimbra&0.732$\pm$0.079&0.733$\pm$0.078&0.749$\pm$0.059&0.750$\pm$0.073&0.732$\pm$0.080&\textbf{0.784$\pm$0.083}&0.767$\pm$0.037\\
diabetic&0.706$\pm$0.019&0.729$\pm$0.018&0.735$\pm$0.030&\textbf{0.739$\pm$0.024}&0.729$\pm$0.021&0.732$\pm$0.016&0.724$\pm$0.017\\
fertility&0.880$\pm$0.027&0.870$\pm$0.027&0.880$\pm$0.027&0.880$\pm$0.027&0.880$\pm$0.027&\textbf{0.890$\pm$0.042}&\textbf{0.890$\pm$0.042}\\
haberman&0.745$\pm$0.071&0.751$\pm$0.072&0.755$\pm$0.055&0.752$\pm$0.061&0.751$\pm$0.070&0.751$\pm$0.047&\textbf{0.768$\pm$0.063}\\
heart&0.839$\pm$0.045&0.839$\pm$0.050&0.839$\pm$0.054&\textbf{0.843$\pm$0.034}&\textbf{0.843$\pm$0.054}&0.839$\pm$0.037&0.839$\pm$0.042\\
monk&0.840$\pm$0.030&0.803$\pm$0.011&0.863$\pm$0.050&\textbf{0.889$\pm$0.037}&0.856$\pm$0.042&0.854$\pm$0.050&0.858$\pm$0.072\\
pima&0.770$\pm$0.031&0.763$\pm$0.044&0.773$\pm$0.046&0.777$\pm$0.045&0.768$\pm$0.042&0.780$\pm$0.030&\textbf{0.786$\pm$0.040}\\
plrx&0.714$\pm$0.126&0.714$\pm$0.126&0.714$\pm$0.126&0.725$\pm$0.133&0.720$\pm$0.124&0.726$\pm$0.130&\textbf{0.742$\pm$0.095}\\
pop failures&0.944$\pm$0.020&\textbf{0.965$\pm$0.010}&\textbf{0.965$\pm$0.010}&0.961$\pm$0.018&0.937$\pm$0.029&0.952$\pm$0.021&0.944$\pm$0.022\\
sonar&0.750$\pm$0.058&0.770$\pm$0.063&0.779$\pm$0.060&\textbf{0.789$\pm$0.082}&0.779$\pm$0.081&\textbf{0.789$\pm$0.086}&0.779$\pm$0.082\\
titanic&0.776$\pm$0.016&0.778$\pm$0.018&0.777$\pm$0.018&0.776$\pm$0.015&0.778$\pm$0.017&0.778$\pm$0.017&\textbf{0.780$\pm$0.017}\\
knowledge&0.968$\pm$0.038&0.983$\pm$0.014&\textbf{0.990$\pm$0.010}&0.988$\pm$0.009&0.963$\pm$0.043&0.980$\pm$0.019&0.988$\pm$0.009\\
wisconsin&0.966$\pm$0.009&0.970$\pm$0.009&0.970$\pm$0.011&0.970$\pm$0.006&0.970$\pm$0.006&0.970$\pm$0.006&\textbf{0.973$\pm$0.009}\\

			\midrule 
			\multicolumn{8}{c}{(b) $25\%$ label noise} \\ 
			\midrule 
			dataset& Pin-SVM & ALS-SVM & EN-SVM & BQ-SVM & BALS-SVM & BAEN-SVM  & $\epsilon$-BAEN-SVM  \\
			\midrule 
australian&0.862$\pm$0.036&0.864$\pm$0.019&0.857$\pm$0.027&\textbf{0.875$\pm$0.017}&0.871$\pm$0.020&0.874$\pm$0.015&0.861$\pm$0.035\\
blood&0.773$\pm$0.030&0.775$\pm$0.034&0.775$\pm$0.034&0.774$\pm$0.028&0.778$\pm$0.039&0.777$\pm$0.029&\textbf{0.786$\pm$0.042}\\
coimbra&0.716$\pm$0.067&0.733$\pm$0.120&0.723$\pm$0.102&0.725$\pm$0.115&0.715$\pm$0.060&0.725$\pm$0.093&\textbf{0.775$\pm$0.073}\\
diabetic&0.644$\pm$0.042&0.661$\pm$0.034&0.665$\pm$0.027&0.674$\pm$0.021&0.670$\pm$0.024&\textbf{0.684$\pm$0.022}&0.663$\pm$0.009\\
fertility&0.810$\pm$0.102&0.810$\pm$0.102&0.830$\pm$0.125&\textbf{0.890$\pm$0.042}&0.880$\pm$0.045&\textbf{0.890$\pm$0.042}&\textbf{0.890$\pm$0.042}\\
haberman&0.745$\pm$0.073&0.755$\pm$0.061&0.755$\pm$0.061&0.751$\pm$0.075&0.761$\pm$0.065&0.755$\pm$0.072&\textbf{0.765$\pm$0.051}\\
heart&0.803$\pm$0.046&0.796$\pm$0.048&0.809$\pm$0.040&0.813$\pm$0.041&0.796$\pm$0.057&\textbf{0.819$\pm$0.056}&0.809$\pm$0.056\\
monk&0.826$\pm$0.018&0.807$\pm$0.038&0.805$\pm$0.038&0.835$\pm$0.025&\textbf{0.851$\pm$0.043}&0.826$\pm$0.018&0.842$\pm$0.039\\
pima&0.775$\pm$0.031&0.767$\pm$0.026&0.775$\pm$0.024&\textbf{0.780$\pm$0.033}&0.776$\pm$0.029&0.779$\pm$0.027&\textbf{0.780$\pm$0.042}\\
plrx&0.714$\pm$0.126&0.687$\pm$0.118&0.720$\pm$0.142&0.726$\pm$0.106&0.725$\pm$0.133&\textbf{0.736$\pm$0.099}&0.726$\pm$0.130\\
pop failures&0.915$\pm$0.022&0.909$\pm$0.029&0.919$\pm$0.025&0.915$\pm$0.022&0.919$\pm$0.023&0.919$\pm$0.023&\textbf{0.922$\pm$0.011}\\
sonar&0.736$\pm$0.075&0.745$\pm$0.076&0.760$\pm$0.044&\textbf{0.764$\pm$0.093}&0.760$\pm$0.079&\textbf{0.764$\pm$0.064}&0.755$\pm$0.097\\
titanic&0.776$\pm$0.016&0.781$\pm$0.021&0.782$\pm$0.019&0.780$\pm$0.014&0.781$\pm$0.019&0.781$\pm$0.020&\textbf{0.783$\pm$0.019}\\
knowledge&0.933$\pm$0.040&0.943$\pm$0.046&0.940$\pm$0.022&0.950$\pm$0.047&0.935$\pm$0.041&0.945$\pm$0.043&\textbf{0.963$\pm$0.029}\\
wisconsin&0.926$\pm$0.006&0.950$\pm$0.010&0.960$\pm$0.011&\textbf{0.970$\pm$0.006}&\textbf{0.970$\pm$0.006}&0.969$\pm$0.006&\textbf{0.970$\pm$0.011}\\

			\midrule
			\multicolumn{8}{c}{(c) $25\%$ feature noise} \\ 
			\midrule 
			dataset& Pin-SVM & ALS-SVM & EN-SVM & BQ-SVM & BALS-SVM & BAEN-SVM & $\epsilon$-BAEN-SVM   \\
			\midrule 
australian&0.867$\pm$0.043&0.878$\pm$0.016&0.878$\pm$0.022&0.874$\pm$0.014&\textbf{0.880$\pm$0.025}&0.875$\pm$0.024&0.877$\pm$0.021\\
blood&0.767$\pm$0.024&0.777$\pm$0.027&0.779$\pm$0.039&0.771$\pm$0.030&0.777$\pm$0.027&0.775$\pm$0.028&\textbf{0.781$\pm$0.043}\\
coimbra&0.741$\pm$0.083&0.750$\pm$0.089&0.750$\pm$0.058&\textbf{0.758$\pm$0.101}&\textbf{0.758$\pm$0.068}&0.749$\pm$0.080&0.750$\pm$0.085\\
diabetic&0.648$\pm$0.047&0.653$\pm$0.044&0.655$\pm$0.053&0.652$\pm$0.055&0.650$\pm$0.044&0.649$\pm$0.045&\textbf{0.669$\pm$0.034}\\
fertility&0.880$\pm$0.027&0.880$\pm$0.027&0.880$\pm$0.027&\textbf{0.890$\pm$0.022}&0.880$\pm$0.027&\textbf{0.890$\pm$0.042}&\textbf{0.890$\pm$0.022}\\
haberman&0.738$\pm$0.074&0.745$\pm$0.084&0.748$\pm$0.073&0.751$\pm$0.079&0.768$\pm$0.063&0.751$\pm$0.063&\textbf{0.774$\pm$0.057}\\
heart&0.836$\pm$0.053&0.839$\pm$0.043&0.839$\pm$0.037&0.836$\pm$0.039&0.843$\pm$0.046&0.843$\pm$0.035&\textbf{0.846$\pm$0.036}\\
monk&0.803$\pm$0.025&0.803$\pm$0.052&0.807$\pm$0.040&0.828$\pm$0.032&0.828$\pm$0.036&\textbf{0.833$\pm$0.020}&\textbf{0.833$\pm$0.038}\\
pima&0.770$\pm$0.047&0.767$\pm$0.041&0.768$\pm$0.037&0.773$\pm$0.059&0.768$\pm$0.042&0.772$\pm$0.046&\textbf{0.775$\pm$0.036}\\
plrx&0.714$\pm$0.126&0.720$\pm$0.132&0.720$\pm$0.132&0.725$\pm$0.115&0.736$\pm$0.059&0.736$\pm$0.094&\textbf{0.742$\pm$0.090}\\
pop failures&0.926$\pm$0.017&0.941$\pm$0.030&\textbf{0.950$\pm$0.011}&0.944$\pm$0.023&0.930$\pm$0.035&0.926$\pm$0.022&0.931$\pm$0.032\\
sonar&0.746$\pm$0.099&0.755$\pm$0.026&0.760$\pm$0.043&0.775$\pm$0.094&0.765$\pm$0.071&0.779$\pm$0.109&\textbf{0.789$\pm$0.098}\\
titanic&0.776$\pm$0.016&0.777$\pm$0.016&0.776$\pm$0.016&0.777$\pm$0.016&0.777$\pm$0.016&\textbf{0.778$\pm$0.017}&\textbf{0.778$\pm$0.018}\\
knowledge&0.963$\pm$0.029&0.973$\pm$0.027&\textbf{0.980$\pm$0.007}&0.963$\pm$0.032&0.965$\pm$0.034&0.970$\pm$0.027&0.975$\pm$0.023\\
wisconsin&0.966$\pm$0.006&0.967$\pm$0.004&\textbf{0.970$\pm$0.003}&\textbf{0.970$\pm$0.006}&\textbf{0.970$\pm$0.006}&\textbf{0.970$\pm$0.003}&\textbf{0.970$\pm$0.006}\\

			\midrule
		\end{tabularx}
	\end{table}

        \begin{table}[htbp]
		\centering
		\fontsize{10pt}{5pt}\selectfont
		\caption{Comparison of the mean $F_{1}$-score ($F_{1}\pm sd$) of seven SVMs with linear kernel in benchmark datasets.}
		\label{tab:res classification linear F1}
		\begin{tabularx}{\hsize}{@{}@{\extracolsep{\fill}}ZZZZZZZZ@{}}
        
			\toprule 
			\multicolumn{8}{c}{(a) $0\%$ noise} \\ 
			\midrule 
			  dataset& Pin-SVM & ALS-SVM & EN-SVM & BQ-SVM & BALS-SVM & BAEN-SVM &$\epsilon$-BAEN-SVM \\
			\midrule 
australian&0.861$\pm$0.038&0.888$\pm$0.020&0.878$\pm$0.021&0.888$\pm$0.014&\textbf{0.891$\pm$0.015}&\textbf{0.891$\pm$0.015}&0.879$\pm$0.022\\
blood&0.867$\pm$0.017&0.871$\pm$0.018&0.871$\pm$0.019&0.869$\pm$0.019&0.871$\pm$0.018&0.871$\pm$0.018&\textbf{0.874$\pm$0.027}\\
coimbra&0.713$\pm$0.084&0.722$\pm$0.068&0.738$\pm$0.085&0.746$\pm$0.086&0.726$\pm$0.086&\textbf{0.772$\pm$0.089}&0.757$\pm$0.083\\
diabetic&0.735$\pm$0.025&0.744$\pm$0.028&0.744$\pm$0.019&\textbf{0.751$\pm$0.030}&0.750$\pm$0.034&0.748$\pm$0.029&0.747$\pm$0.014\\
fertility&0.936$\pm$0.016&0.930$\pm$0.016&0.936$\pm$0.016&0.936$\pm$0.016&0.936$\pm$0.016&\textbf{0.941$\pm$0.023}&\textbf{0.941$\pm$0.023}\\
haberman&0.849$\pm$0.055&0.848$\pm$0.054&0.849$\pm$0.056&0.849$\pm$0.055&0.850$\pm$0.046&0.852$\pm$0.038&\textbf{0.855$\pm$0.065}\\
heart&0.884$\pm$0.033&0.886$\pm$0.035&0.886$\pm$0.039&0.887$\pm$0.026&\textbf{0.888$\pm$0.036}&0.886$\pm$0.024&0.886$\pm$0.030\\
monk&0.830$\pm$0.030&0.790$\pm$0.015&0.866$\pm$0.053&\textbf{0.881$\pm$0.042}&0.845$\pm$0.041&0.845$\pm$0.046&0.848$\pm$0.043\\
pima&0.835$\pm$0.025&0.829$\pm$0.030&0.838$\pm$0.036&0.840$\pm$0.034&0.836$\pm$0.039&0.842$\pm$0.021&\textbf{0.847$\pm$0.033}\\
plrx&0.828$\pm$0.088&0.828$\pm$0.088&0.828$\pm$0.088&\textbf{0.834$\pm$0.091}&0.831$\pm$0.092&\textbf{0.834$\pm$0.090}&\textbf{0.834$\pm$0.091}\\
pop failures&0.572$\pm$0.084&0.743$\pm$0.115&0.748$\pm$0.047&\textbf{0.761$\pm$0.129}&0.413$\pm$0.294&0.610$\pm$0.138&0.622$\pm$0.124\\
sonar&0.771$\pm$0.062&0.791$\pm$0.055&\textbf{0.800$\pm$0.054}&0.799$\pm$0.069&0.787$\pm$0.090&0.788$\pm$0.103&0.778$\pm$0.080\\
titanic&0.847$\pm$0.012&0.855$\pm$0.014&0.850$\pm$0.014&0.847$\pm$0.012&0.848$\pm$0.013&0.848$\pm$0.013&\textbf{0.856$\pm$0.015}\\
knowledge&0.966$\pm$0.045&0.983$\pm$0.015&\textbf{0.991$\pm$0.009}&0.989$\pm$0.007&0.961$\pm$0.052&0.981$\pm$0.021&0.989$\pm$0.007\\
wisconsin&0.974$\pm$0.007&0.977$\pm$0.007&0.977$\pm$0.004&0.977$\pm$0.007&0.977$\pm$0.006&0.977$\pm$0.006&\textbf{0.979$\pm$0.006}\\

			\midrule
			\multicolumn{8}{c}{(b) $25\%$ label noise} \\ 
			\midrule 
			dataset& Pin-SVM & ALS-SVM & EN-SVM & BQ-SVM & BALS-SVM & BAEN-SVM &$\epsilon$-BAEN-SVM  \\
			\midrule 
australian&0.868$\pm$0.037&0.874$\pm$0.021&0.866$\pm$0.029&\textbf{0.888$\pm$0.015}&0.884$\pm$0.009&0.887$\pm$0.011&0.866$\pm$0.035\\
blood&0.869$\pm$0.019&0.869$\pm$0.024&0.869$\pm$0.022&0.869$\pm$0.019&0.871$\pm$0.024&0.870$\pm$0.018&\textbf{0.874$\pm$0.026}\\
coimbra&0.732$\pm$0.052&0.712$\pm$0.065&0.701$\pm$0.077&0.725$\pm$0.089&0.708$\pm$0.097&0.731$\pm$0.084&\textbf{0.744$\pm$0.085}\\
diabetic&0.663$\pm$0.044&0.653$\pm$0.033&0.660$\pm$0.027&0.719$\pm$0.010&0.684$\pm$0.035&\textbf{0.721$\pm$0.011}&0.686$\pm$0.024\\
fertility&0.888$\pm$0.098&0.884$\pm$0.072&0.902$\pm$0.075&\textbf{0.941$\pm$0.023}&0.936$\pm$0.016&\textbf{0.941$\pm$0.023}&\textbf{0.941$\pm$0.023}\\
haberman&0.846$\pm$0.052&0.848$\pm$0.050&0.848$\pm$0.043&0.850$\pm$0.054&0.851$\pm$0.048&0.850$\pm$0.052&\textbf{0.855$\pm$0.050}\\
heart&0.859$\pm$0.044&0.849$\pm$0.034&0.861$\pm$0.031&0.870$\pm$0.025&0.854$\pm$0.050&\textbf{0.871$\pm$0.044}&0.865$\pm$0.043\\
monk&0.814$\pm$0.016&0.794$\pm$0.057&0.798$\pm$0.040&0.826$\pm$0.022&\textbf{0.840$\pm$0.041}&0.819$\pm$0.032&0.827$\pm$0.042\\
pima&0.837$\pm$0.027&0.830$\pm$0.023&0.838$\pm$0.025&0.841$\pm$0.028&0.839$\pm$0.025&0.840$\pm$0.024&\textbf{0.843$\pm$0.033}\\
plrx&0.828$\pm$0.088&0.801$\pm$0.087&0.830$\pm$0.095&0.831$\pm$0.092&0.834$\pm$0.091&\textbf{0.837$\pm$0.075}&0.834$\pm$0.090\\
pop failures&0.234$\pm$0.147&0.216$\pm$0.063&0.229$\pm$0.057&0.265$\pm$0.241&0.262$\pm$0.028&0.272$\pm$0.123&\textbf{0.424$\pm$0.208}\\
sonar&0.734$\pm$0.100&0.743$\pm$0.095&0.760$\pm$0.063&0.762$\pm$0.067&0.760$\pm$0.079&0.765$\pm$0.099&\textbf{0.772$\pm$0.078}\\
titanic&0.847$\pm$0.012&0.850$\pm$0.013&0.850$\pm$0.013&0.849$\pm$0.011&0.850$\pm$0.013&0.851$\pm$0.014&\textbf{0.858$\pm$0.016}\\
knowledge&0.932$\pm$0.049&0.946$\pm$0.022&0.946$\pm$0.022&0.949$\pm$0.056&0.936$\pm$0.050&0.944$\pm$0.053&\textbf{0.964$\pm$0.031}\\
wisconsin&0.945$\pm$0.006&0.962$\pm$0.009&0.970$\pm$0.008&\textbf{0.977$\pm$0.006}&\textbf{0.977$\pm$0.006}&0.976$\pm$0.006&0.976$\pm$0.009\\

			\midrule 
			\multicolumn{8}{c}{(c) $25\%$ feature noise} \\ 
			\midrule 
			dataset & Pin-SVM & ALS-SVM & EN-SVM & BQ-SVM & BALS-SVM & BAEN-SVM &$\epsilon$-BAEN-SVM \\ 
			\midrule 
australian&0.874$\pm$0.026&0.888$\pm$0.016&0.887$\pm$0.023&0.886$\pm$0.012&\textbf{0.889$\pm$0.026}&0.887$\pm$0.012&0.884$\pm$0.025\\
blood&0.867$\pm$0.016&0.871$\pm$0.018&0.870$\pm$0.017&0.867$\pm$0.018&0.871$\pm$0.018&0.870$\pm$0.018&\textbf{0.872$\pm$0.018}\\
coimbra&0.736$\pm$0.084&0.728$\pm$0.085&0.750$\pm$0.065&\textbf{0.760$\pm$0.099}&0.740$\pm$0.073&0.754$\pm$0.080&0.742$\pm$0.074\\
diabetic&0.659$\pm$0.077&0.635$\pm$0.052&0.665$\pm$0.059&0.678$\pm$0.035&0.675$\pm$0.031&\textbf{0.690$\pm$0.015}&0.685$\pm$0.032\\
fertility&0.936$\pm$0.016&0.936$\pm$0.016&0.936$\pm$0.016&\textbf{0.941$\pm$0.012}&0.936$\pm$0.016&\textbf{0.941$\pm$0.023}&\textbf{0.941$\pm$0.023}\\
haberman&0.846$\pm$0.059&0.845$\pm$0.058&0.848$\pm$0.048&0.849$\pm$0.056&0.852$\pm$0.052&0.849$\pm$0.051&\textbf{0.856$\pm$0.043}\\
heart&0.883$\pm$0.036&0.885$\pm$0.028&0.886$\pm$0.025&0.883$\pm$0.029&0.886$\pm$0.030&0.888$\pm$0.026&\textbf{0.893$\pm$0.021}\\
monk&0.792$\pm$0.031&0.789$\pm$0.060&0.793$\pm$0.044&0.815$\pm$0.028&0.818$\pm$0.035&0.823$\pm$0.047&\textbf{0.825$\pm$0.038}\\
pima&0.836$\pm$0.036&0.835$\pm$0.031&0.834$\pm$0.029&\textbf{0.839$\pm$0.035}&0.835$\pm$0.032&\textbf{0.839$\pm$0.035}&\textbf{0.839$\pm$0.031}\\
plrx&0.828$\pm$0.088&0.831$\pm$0.091&0.831$\pm$0.091&0.831$\pm$0.091&0.831$\pm$0.091&\textbf{0.834$\pm$0.086}&\textbf{0.834$\pm$0.094}\\
pop failures&0.425$\pm$0.126&0.495$\pm$0.162&\textbf{0.644$\pm$0.145}&0.541$\pm$0.189&0.361$\pm$0.352&0.463$\pm$0.175&0.503$\pm$0.126\\
sonar&0.765$\pm$0.093&0.777$\pm$0.066&0.778$\pm$0.043&\textbf{0.792$\pm$0.019}&0.770$\pm$0.085&0.777$\pm$0.078&0.786$\pm$0.123\\
titanic&0.847$\pm$0.012&0.847$\pm$0.016&0.851$\pm$0.015&0.847$\pm$0.012&0.847$\pm$0.012&0.848$\pm$0.013&\textbf{0.852$\pm$0.013}\\
knowledge&0.964$\pm$0.031&0.973$\pm$0.031&\textbf{0.982$\pm$0.007}&0.963$\pm$0.036&0.965$\pm$0.038&0.972$\pm$0.023&0.978$\pm$0.019\\
wisconsin&0.974$\pm$0.005&0.975$\pm$0.004&\textbf{0.977$\pm$0.004}&\textbf{0.977$\pm$0.006}&\textbf{0.977$\pm$0.006}&\textbf{0.977$\pm$0.004}&\textbf{0.977$\pm$0.006}\\

			\midrule
		\end{tabularx}
	\end{table}

   \begin{table}[htbp]
		\centering
		\fontsize{10pt}{5pt}\selectfont
		\caption{Comparison of the mean accuracy (ACC$\pm$sd) of seven SVMs with RBF kernel in benchmark datasets.}
		\label{tab:res classification Gaussian acc}
		\begin{tabularx}{\hsize}{@{}@{\extracolsep{\fill}}ZZZZZZZZ@{}}
			\toprule 
			\multicolumn{8}{c}{(a) $0\%$ noise} \\ 
			\midrule 
			dataset& Pin-SVM & ALS-SVM & EN-SVM & BQ-SVM & BALS-SVM & BAEN-SVM &$\epsilon$-BAEN-SVM \\
			\midrule 
australian&0.871$\pm$0.024&\textbf{0.872$\pm$0.022}&0.871$\pm$0.018&\textbf{0.872$\pm$0.026}&\textbf{0.872$\pm$0.022}&\textbf{0.872$\pm$0.022}&0.871$\pm$0.014\\
blood&0.794$\pm$0.042&0.797$\pm$0.050&0.799$\pm$0.045&0.795$\pm$0.043&0.798$\pm$0.040&0.799$\pm$0.034&\textbf{0.801$\pm$0.032}\\
coimbra&0.758$\pm$0.068&0.758$\pm$0.041&\textbf{0.784$\pm$0.088}&0.758$\pm$0.087&0.758$\pm$0.041&0.767$\pm$0.073&0.783$\pm$0.083\\
diabetic&0.732$\pm$0.028&0.730$\pm$0.020&\textbf{0.735$\pm$0.024}&0.712$\pm$0.028&0.719$\pm$0.029&0.721$\pm$0.034&0.720$\pm$0.029\\
fertility&0.880$\pm$0.027&0.890$\pm$0.022&0.890$\pm$0.022&0.880$\pm$0.027&0.880$\pm$0.027&0.900$\pm$0.035&\textbf{0.910$\pm$0.065}\\
haberman&0.755$\pm$0.063&0.761$\pm$0.076&0.758$\pm$0.073&0.765$\pm$0.066&0.761$\pm$0.076&\textbf{0.768$\pm$0.060}&0.764$\pm$0.064\\
heart&0.802$\pm$0.044&0.809$\pm$0.046&0.819$\pm$0.058&\textbf{0.829$\pm$0.033}&0.809$\pm$0.046&0.813$\pm$0.044&0.823$\pm$0.042\\
monk&0.979$\pm$0.015&0.979$\pm$0.017&\textbf{1.000$\pm$0.000}&\textbf{1.000$\pm$0.000}&0.975$\pm$0.013&0.981$\pm$0.013&0.995$\pm$0.006\\
pima&0.767$\pm$0.045&0.768$\pm$0.051&0.770$\pm$0.042&0.768$\pm$0.039&0.769$\pm$0.051&\textbf{0.773$\pm$0.061}&0.772$\pm$0.042\\
plrx&0.725$\pm$0.133&0.725$\pm$0.133&0.725$\pm$0.133&0.725$\pm$0.133&0.725$\pm$0.133&0.725$\pm$0.133&\textbf{0.731$\pm$0.076}\\
pop failures&0.944$\pm$0.020&0.944$\pm$0.020&0.948$\pm$0.017&0.948$\pm$0.017&0.944$\pm$0.020&\textbf{0.950$\pm$0.018}&\textbf{0.950$\pm$0.017}\\
sonar&0.909$\pm$0.045&0.909$\pm$0.045&\textbf{0.914$\pm$0.046}&0.909$\pm$0.045&0.909$\pm$0.045&0.909$\pm$0.056&\textbf{0.914$\pm$0.036}\\
titanic&0.790$\pm$0.018&\textbf{0.791$\pm$0.018}&0.790$\pm$0.018&0.790$\pm$0.018&\textbf{0.791$\pm$0.018}&\textbf{0.791$\pm$0.018}&\textbf{0.791$\pm$0.018}\\
knowledge&0.968$\pm$0.026&0.980$\pm$0.011&0.980$\pm$0.007&0.978$\pm$0.020&0.975$\pm$0.023&0.973$\pm$0.021&\textbf{0.983$\pm$0.007}\\
wisconsin&0.973$\pm$0.006&0.974$\pm$0.010&0.974$\pm$0.011&0.974$\pm$0.008&0.974$\pm$0.010&0.974$\pm$0.010&\textbf{0.976$\pm$0.008}\\

			\midrule 
			\multicolumn{8}{c}{(b) $25\%$ label noise} \\ 
			\midrule 
			dataset & Pin-SVM & ALS-SVM & EN-SVM & BQ-SVM & BALS-SVM & BAEN-SVM &$\epsilon$-BAEN-SVM \\ 
			\midrule 
australian&0.895$\pm$0.085&0.895$\pm$0.078&0.895$\pm$0.078&0.905$\pm$0.075&0.905$\pm$0.075&\textbf{0.914$\pm$0.062}&\textbf{0.914$\pm$0.062}\\
blood&0.782$\pm$0.049&0.781$\pm$0.045&0.785$\pm$0.043&0.791$\pm$0.040&0.790$\pm$0.017&\textbf{0.793$\pm$0.042}&0.789$\pm$0.059\\
coimbra&0.733$\pm$0.048&\textbf{0.742$\pm$0.067}&0.741$\pm$0.075&0.724$\pm$0.037&0.733$\pm$0.077&0.741$\pm$0.053&\textbf{0.742$\pm$0.059}\\
diabetic&0.662$\pm$0.037&0.663$\pm$0.024&0.663$\pm$0.024&0.658$\pm$0.036&\textbf{0.665$\pm$0.028}&\textbf{0.665$\pm$0.026}&\textbf{0.665$\pm$0.026}\\
fertility&0.880$\pm$0.027&0.880$\pm$0.027&0.880$\pm$0.027&\textbf{0.890$\pm$0.042}&\textbf{0.890$\pm$0.065}&\textbf{0.890$\pm$0.114}&\textbf{0.890$\pm$0.096}\\
haberman&0.761$\pm$0.061&0.758$\pm$0.068&0.761$\pm$0.069&0.768$\pm$0.064&0.764$\pm$0.067&\textbf{0.771$\pm$0.043}&0.768$\pm$0.072\\
heart&0.782$\pm$0.069&0.776$\pm$0.070&0.789$\pm$0.060&0.789$\pm$0.051&0.789$\pm$0.054&0.793$\pm$0.080&\textbf{0.799$\pm$0.052}\\
monk&0.944$\pm$0.041&0.921$\pm$0.038&0.942$\pm$0.026&\textbf{0.972$\pm$0.013}&0.970$\pm$0.013&\textbf{0.972$\pm$0.013}&\textbf{0.972$\pm$0.013}\\
pima&0.762$\pm$0.052&0.766$\pm$0.040&0.767$\pm$0.046&0.767$\pm$0.038&0.767$\pm$0.040&\textbf{0.768$\pm$0.049}&0.767$\pm$0.040\\
plrx&0.725$\pm$0.133&0.725$\pm$0.133&0.725$\pm$0.133&0.725$\pm$0.133&0.725$\pm$0.133&0.725$\pm$0.133&\textbf{0.731$\pm$0.137}\\
pop failures&0.915$\pm$0.022&0.919$\pm$0.028&0.919$\pm$0.028&0.915$\pm$0.022&\textbf{0.920$\pm$0.024}&\textbf{0.920$\pm$0.024}&\textbf{0.920$\pm$0.024}\\
sonar&0.807$\pm$0.094&0.802$\pm$0.116&\textbf{0.826$\pm$0.095}&0.812$\pm$0.100&0.822$\pm$0.089&\textbf{0.826$\pm$0.095}&0.822$\pm$0.089\\
titanic&0.775$\pm$0.018&0.780$\pm$0.017&0.779$\pm$0.018&0.777$\pm$0.018&0.783$\pm$0.019&0.783$\pm$0.017&\textbf{0.785$\pm$0.022}\\
knowledge&0.943$\pm$0.030&0.940$\pm$0.032&0.950$\pm$0.026&0.955$\pm$0.030&0.948$\pm$0.038&\textbf{0.958$\pm$0.029}&\textbf{0.958$\pm$0.036}\\
wisconsin&0.966$\pm$0.003&0.963$\pm$0.006&0.967$\pm$0.008&0.970$\pm$0.008&0.970$\pm$0.008&0.970$\pm$0.006&\textbf{0.971$\pm$0.007}\\

			\midrule
			\multicolumn{8}{c}{(c) $25\%$ feature noise} \\ 
			\midrule 
			dataset & Pin-SVM & ALS-SVM & EN-SVM & BQ-SVM & BALS-SVM & BAEN-SVM &$\epsilon$-BAEN-SVM \\ 
			\midrule 
australian&0.877$\pm$0.073&0.886$\pm$0.087&0.886$\pm$0.087&0.886$\pm$0.064&0.886$\pm$0.064&0.895$\pm$0.078&\textbf{0.905$\pm$0.058}\\
blood&0.770$\pm$0.046&0.779$\pm$0.041&\textbf{0.783$\pm$0.057}&\textbf{0.783$\pm$0.057}&0.779$\pm$0.050&0.781$\pm$0.038&\textbf{0.783$\pm$0.028}\\
coimbra&0.707$\pm$0.095&0.707$\pm$0.078&0.757$\pm$0.102&0.741$\pm$0.088&0.759$\pm$0.079&\textbf{0.767$\pm$0.068}&0.749$\pm$0.085\\
diabetic&0.665$\pm$0.042&0.672$\pm$0.036&0.676$\pm$0.036&0.671$\pm$0.031&0.672$\pm$0.037&0.674$\pm$0.032&\textbf{0.678$\pm$0.038}\\
fertility&\textbf{0.900$\pm$0.035}&0.890$\pm$0.022&0.890$\pm$0.022&\textbf{0.900$\pm$0.035}&\textbf{0.900$\pm$0.035}&\textbf{0.900$\pm$0.035}&\textbf{0.900$\pm$0.035}\\
haberman&0.742$\pm$0.066&0.758$\pm$0.068&0.758$\pm$0.084&0.758$\pm$0.064&0.755$\pm$0.078&\textbf{0.764$\pm$0.079}&0.758$\pm$0.076\\
heart&0.799$\pm$0.069&0.806$\pm$0.077&0.809$\pm$0.068&0.812$\pm$0.068&0.809$\pm$0.074&0.809$\pm$0.078&\textbf{0.816$\pm$0.073}\\
monk&\textbf{0.956$\pm$0.015}&0.954$\pm$0.020&\textbf{0.956$\pm$0.013}&0.954$\pm$0.028&0.947$\pm$0.024&0.949$\pm$0.019&0.947$\pm$0.021\\
pima&0.767$\pm$0.042&\textbf{0.771$\pm$0.043}&\textbf{0.771$\pm$0.043}&0.768$\pm$0.042&\textbf{0.771$\pm$0.043}&0.768$\pm$0.042&\textbf{0.771$\pm$0.047}\\
plrx&0.725$\pm$0.133&0.725$\pm$0.133&0.725$\pm$0.133&0.725$\pm$0.133&0.725$\pm$0.133&0.726$\pm$0.124&\textbf{0.731$\pm$0.131}\\
pop failures&0.928$\pm$0.029&0.928$\pm$0.028&0.933$\pm$0.032&0.930$\pm$0.031&0.930$\pm$0.029&\textbf{0.935$\pm$0.017}&\textbf{0.935$\pm$0.032}\\
sonar&0.842$\pm$0.074&0.837$\pm$0.084&0.866$\pm$0.080&0.842$\pm$0.074&0.842$\pm$0.074&0.851$\pm$0.073&\textbf{0.875$\pm$0.072}\\
titanic&0.785$\pm$0.017&0.788$\pm$0.018&\textbf{0.789$\pm$0.020}&0.788$\pm$0.016&0.785$\pm$0.016&0.788$\pm$0.016&0.787$\pm$0.019\\
knowledge&0.963$\pm$0.035&0.965$\pm$0.031&0.965$\pm$0.031&0.968$\pm$0.034&0.965$\pm$0.032&0.970$\pm$0.030&\textbf{0.975$\pm$0.023}\\
wisconsin&0.969$\pm$0.004&0.970$\pm$0.006&0.970$\pm$0.006&\textbf{0.971$\pm$0.007}&0.970$\pm$0.008&0.970$\pm$0.008&\textbf{0.971$\pm$0.007}\\

			\midrule
		\end{tabularx}
	\end{table}

  \begin{table}[htp]
		\centering
		\fontsize{10pt}{5pt}\selectfont
		\caption{Comparison of the mean $F_{1}$-score ($F_{1}\pm sd$) of seven SVMs with RBF kernel in benchmark datasets.}
		\label{tab:res classification Gaussian F1}
		\begin{tabularx}{\hsize}{@{}@{\extracolsep{\fill}}ZZZZZZZZ@{}}
			\toprule 
			\multicolumn{8}{c}{(a) $0\%$ noise} \\ 
			\midrule 
			  dataset& Pin-SVM & ALS-SVM & EN-SVM & BQ-SVM & BALS-SVM & BAEN-SVM & $\epsilon$-BAEN-SVM   \\
			\midrule 
australian&0.882$\pm$0.022&\textbf{0.886$\pm$0.016}&0.885$\pm$0.013&0.884$\pm$0.012&0.885$\pm$0.016&\textbf{0.886$\pm$0.011}&0.884$\pm$0.008\\
blood&0.873$\pm$0.028&0.876$\pm$0.033&0.876$\pm$0.030&0.876$\pm$0.031&0.876$\pm$0.027&\textbf{0.878$\pm$0.021}&0.877$\pm$0.024\\
coimbra&0.736$\pm$0.069&0.736$\pm$0.031&0.765$\pm$0.076&0.732$\pm$0.048&0.736$\pm$0.031&0.751$\pm$0.078&\textbf{0.769$\pm$0.075}\\
diabetic&\textbf{0.742$\pm$0.029}&0.727$\pm$0.035&0.729$\pm$0.022&0.731$\pm$0.026&0.727$\pm$0.027&0.739$\pm$0.018&0.732$\pm$0.027\\
fertility&0.936$\pm$0.016&0.941$\pm$0.012&0.941$\pm$0.012&0.936$\pm$0.016&0.936$\pm$0.016&0.946$\pm$0.019&\textbf{0.950$\pm$0.036}\\
haberman&0.847$\pm$0.042&0.851$\pm$0.052&0.849$\pm$0.052&0.856$\pm$0.061&0.853$\pm$0.042&\textbf{0.858$\pm$0.040}&0.854$\pm$0.046\\
heart&0.863$\pm$0.039&0.869$\pm$0.031&0.872$\pm$0.041&\textbf{0.880$\pm$0.023}&0.869$\pm$0.031&0.870$\pm$0.050&0.873$\pm$0.027\\
monk&0.978$\pm$0.016&0.979$\pm$0.018&\textbf{1.000$\pm$0.000}&\textbf{1.000$\pm$0.000}&0.974$\pm$0.013&0.980$\pm$0.014&0.995$\pm$0.007\\
pima&0.836$\pm$0.036&0.835$\pm$0.037&0.837$\pm$0.037&0.837$\pm$0.037&0.837$\pm$0.038&\textbf{0.843$\pm$0.043}&0.839$\pm$0.032\\
plrx&\textbf{0.834$\pm$0.091}&\textbf{0.834$\pm$0.091}&\textbf{0.834$\pm$0.091}&\textbf{0.834$\pm$0.091}&\textbf{0.834$\pm$0.091}&\textbf{0.834$\pm$0.091}&\textbf{0.834$\pm$0.091}\\
pop failures&0.523$\pm$0.220&0.533$\pm$0.150&0.591$\pm$0.153&0.575$\pm$0.179&0.523$\pm$0.220&0.590$\pm$0.188&\textbf{0.601$\pm$0.128}\\
sonar&0.915$\pm$0.045&0.915$\pm$0.045&0.919$\pm$0.046&0.915$\pm$0.045&0.915$\pm$0.045&0.917$\pm$0.042&\textbf{0.922$\pm$0.032}\\
titanic&0.864$\pm$0.012&0.864$\pm$0.012&0.864$\pm$0.012&0.864$\pm$0.012&\textbf{0.865$\pm$0.013}&\textbf{0.865$\pm$0.013}&\textbf{0.865$\pm$0.013}\\
knowledge&0.968$\pm$0.030&0.981$\pm$0.013&0.982$\pm$0.007&0.979$\pm$0.021&0.976$\pm$0.026&0.973$\pm$0.024&\textbf{0.984$\pm$0.007}\\
wisconsin&0.979$\pm$0.008&0.980$\pm$0.008&0.980$\pm$0.009&0.980$\pm$0.004&0.980$\pm$0.004&0.980$\pm$0.008&\textbf{0.981$\pm$0.007}\\

			\midrule
			\multicolumn{8}{c}{(b) $25\%$ label noise} \\ 
			\midrule 
			dataset& Pin-SVM & ALS-SVM & EN-SVM & BQ-SVM & BALS-SVM & BAEN-SVM & $\epsilon$-BAEN-SVM \\
			\midrule 
australian&0.935$\pm$0.053&0.936$\pm$0.046&0.936$\pm$0.046&0.942$\pm$0.045&0.942$\pm$0.045&\textbf{0.947$\pm$0.038}&\textbf{0.947$\pm$0.038}\\
blood&0.870$\pm$0.032&0.870$\pm$0.015&0.872$\pm$0.026&0.873$\pm$0.025&0.872$\pm$0.033&\textbf{0.876$\pm$0.031}&0.873$\pm$0.035\\
coimbra&0.703$\pm$0.057&0.712$\pm$0.111&0.712$\pm$0.111&0.711$\pm$0.052&0.706$\pm$0.109&0.713$\pm$0.109&\textbf{0.721$\pm$0.080}\\
diabetic&0.664$\pm$0.041&0.646$\pm$0.035&0.663$\pm$0.031&0.685$\pm$0.021&0.681$\pm$0.024&0.686$\pm$0.026&\textbf{0.689$\pm$0.029}\\
fertility&0.936$\pm$0.016&0.936$\pm$0.016&0.936$\pm$0.016&\textbf{0.941$\pm$0.023}&0.940$\pm$0.036&\textbf{0.941$\pm$0.023}&\textbf{0.941$\pm$0.023}\\
haberman&0.851$\pm$0.043&0.849$\pm$0.046&0.853$\pm$0.056&0.853$\pm$0.046&0.853$\pm$0.054&\textbf{0.857$\pm$0.033}&0.854$\pm$0.050\\
heart&0.849$\pm$0.040&0.848$\pm$0.047&0.854$\pm$0.041&0.852$\pm$0.040&0.853$\pm$0.039&0.857$\pm$0.022&\textbf{0.861$\pm$0.043}\\
monk&0.943$\pm$0.042&0.917$\pm$0.040&0.939$\pm$0.027&\textbf{0.972$\pm$0.014}&0.969$\pm$0.013&\textbf{0.972$\pm$0.014}&\textbf{0.972$\pm$0.014}\\
pima&0.833$\pm$0.039&0.832$\pm$0.029&0.835$\pm$0.034&\textbf{0.838$\pm$0.031}&0.835$\pm$0.038&0.836$\pm$0.037&0.835$\pm$0.030\\
plrx&0.834$\pm$0.091&0.834$\pm$0.091&0.834$\pm$0.091&0.834$\pm$0.091&0.834$\pm$0.091&0.834$\pm$0.091&\textbf{0.836$\pm$0.093}\\
pop failures&0.224$\pm$0.075&0.247$\pm$0.073&0.277$\pm$0.077&0.326$\pm$0.178&0.296$\pm$0.155&0.326$\pm$0.142&\textbf{0.332$\pm$0.087}\\
sonar&0.817$\pm$0.098&0.810$\pm$0.118&0.835$\pm$0.111&0.823$\pm$0.100&0.834$\pm$0.089&\textbf{0.838$\pm$0.093}&0.835$\pm$0.099\\
titanic&0.851$\pm$0.015&0.855$\pm$0.013&0.855$\pm$0.013&0.852$\pm$0.015&0.858$\pm$0.014&0.857$\pm$0.013&\textbf{0.859$\pm$0.017}\\
knowledge&0.945$\pm$0.032&0.944$\pm$0.033&0.953$\pm$0.027&0.958$\pm$0.028&0.950$\pm$0.039&\textbf{0.960$\pm$0.030}&0.958$\pm$0.042\\
wisconsin&0.974$\pm$0.004&0.972$\pm$0.004&0.975$\pm$0.007&0.977$\pm$0.007&0.977$\pm$0.007&0.977$\pm$0.006&\textbf{0.978$\pm$0.007}\\

			\midrule 
			\multicolumn{8}{c}{(c) $25\%$ feature noise} \\ 
			\midrule 
			dataset& Pin-SVM & ALS-SVM & EN-SVM & BQ-SVM & BALS-SVM & BAEN-SVM & $\epsilon$-BAEN-SVM   \\
			\midrule 
australian&0.926$\pm$0.043&0.931$\pm$0.051&0.931$\pm$0.051&0.931$\pm$0.051&0.931$\pm$0.051&0.936$\pm$0.046&\textbf{0.943$\pm$0.045}\\
blood&0.868$\pm$0.016&0.868$\pm$0.033&0.869$\pm$0.027&\textbf{0.873$\pm$0.032}&0.872$\pm$0.029&\textbf{0.873$\pm$0.028}&\textbf{0.873$\pm$0.025}\\
coimbra&0.709$\pm$0.077&0.702$\pm$0.063&0.750$\pm$0.099&0.744$\pm$0.072&0.767$\pm$0.073&\textbf{0.770$\pm$0.079}&0.756$\pm$0.080\\
diabetic&0.682$\pm$0.045&0.682$\pm$0.034&0.687$\pm$0.043&0.688$\pm$0.045&0.686$\pm$0.038&\textbf{0.698$\pm$0.028}&0.692$\pm$0.037\\
fertility&\textbf{0.946$\pm$0.019}&0.941$\pm$0.012&0.941$\pm$0.012&\textbf{0.946$\pm$0.019}&\textbf{0.946$\pm$0.019}&\textbf{0.946$\pm$0.019}&\textbf{0.946$\pm$0.019}\\
haberman&0.846$\pm$0.050&0.852$\pm$0.050&0.855$\pm$0.061&0.852$\pm$0.055&0.849$\pm$0.052&0.854$\pm$0.051&\textbf{0.857$\pm$0.054}\\
heart&0.862$\pm$0.045&0.868$\pm$0.052&\textbf{0.869$\pm$0.056}&0.868$\pm$0.047&\textbf{0.869$\pm$0.051}&\textbf{0.869$\pm$0.054}&\textbf{0.869$\pm$0.055}\\
monk&0.954$\pm$0.015&0.952$\pm$0.010&\textbf{0.955$\pm$0.011}&0.952$\pm$0.030&0.945$\pm$0.026&0.947$\pm$0.020&0.945$\pm$0.023\\
pima&0.834$\pm$0.034&0.837$\pm$0.036&0.835$\pm$0.035&0.836$\pm$0.044&0.836$\pm$0.034&0.837$\pm$0.042&\textbf{0.838$\pm$0.032}\\
plrx&0.834$\pm$0.091&0.834$\pm$0.091&0.834$\pm$0.091&0.834$\pm$0.091&0.834$\pm$0.091&0.834$\pm$0.091&\textbf{0.836$\pm$0.091}\\
pop failures&0.319$\pm$0.252&0.319$\pm$0.252&0.372$\pm$0.268&0.319$\pm$0.252&0.326$\pm$0.261&\textbf{0.474$\pm$0.155}&0.434$\pm$0.088\\
sonar&0.854$\pm$0.071&0.851$\pm$0.079&0.877$\pm$0.042&0.854$\pm$0.071&0.854$\pm$0.071&0.866$\pm$0.065&\textbf{0.881$\pm$0.071}\\
titanic&0.861$\pm$0.012&0.861$\pm$0.011&0.861$\pm$0.011&0.861$\pm$0.011&0.861$\pm$0.011&0.861$\pm$0.011&\textbf{0.862$\pm$0.012}\\
knowledge&0.963$\pm$0.039&0.966$\pm$0.033&0.966$\pm$0.033&0.969$\pm$0.024&0.966$\pm$0.034&\textbf{0.976$\pm$0.026}&0.971$\pm$0.030\\
wisconsin&0.976$\pm$0.004&0.977$\pm$0.006&0.977$\pm$0.007&\textbf{0.978$\pm$0.006}&0.977$\pm$0.006&0.977$\pm$0.006&\textbf{0.978$\pm$0.006}\\

			\midrule
		\end{tabularx}
	\end{table}

        \subsection{Comparisons by statistical test} 
        In this section, we apply the Friedman test \citep{Janez2006StattestClassifier} to evaluate whether there are statistically significant differences between the seven SVM models across 15 datasets. The null hypothesis of the Friedman test assumes that all models perform equivalently. The test statistic $F_{F}$ follows an $F$ distribution with degrees of freedom $(F(k-1,(k-1)(N-1)))$, where $N=15$ is the number of datasets and $k=7$ is the number of classifiers. The $F_{F}$ statistic is defined as
            \begin{equation}F_{F} = \frac{(N - 1)\chi^{2}_F}{N(k-1)-\chi^{2}_F}.\end{equation}
            where $\chi^{2}_F$ is the raw Friedman statistic, given by 
            \begin{equation}\chi^{2}_F = \frac{12N}{k(k+1)}\bigg{(}\sum_{j=1}^{k}R_{j}^{2} - \frac{k(k+1)^{2}}{4}\bigg{)},\end{equation}
	      where $R_{j}$ is the average rank of the $j$-th classifier. 
        The results for $F_{F}$ and $\chi^{2}_F$ for each type of kernel and noise are listed in \autoref{tab: FF and ChisqF classifier}. At the level of significance of $\alpha = 0.05$, the critical value is $F_{\alpha}(6, 84) = 2.21$. Since all $F_{F}$ values exceed this threshold, we conclude that there are statistically significant differences among the seven SVM models.
	    \begin{table}
		\centering
		\footnotesize
		\caption{The result of Friedman test on seven classifiers.}
		\begin{tabularx}{\hsize}{@{}@{\extracolsep{\fill}}ZZZZZZ@{}}
			\toprule
			Table & Kernel & evaluation index & Noise & $\chi^{2}_F$ & $F_{F}$\\
			\midrule
			\multirow{3}*{\autoref{tab:res classification linear acc}} &             \multirow{3}*{linear} &   
			\multirow{3}*{ACC}   
			& without noise          & 29.04 & 6.67\\
			& & & 25\% label noise   & 39.68 & 11.04\\
			& & & 25\% feature noise & 37.05 & 9.78\\
			\midrule
			\multirow{3}*{\autoref{tab:res classification linear F1}} &
			\multirow{3}*{linear} &   
			\multirow{3}*{$F_{1}$}   
			& without noise          & 30.65 & 7.23\\
			& & & 25\% label noise   & 48.99 & 16.73\\
			& & & 25\% feature noise & 33.46 & 8.28\\
			\midrule
			\multirow{3}*{\autoref{tab:res classification Gaussian acc}} &
			\multirow{3}*{Gaussianl} &   
			\multirow{3}*{ACC}   
			& without noise          & 28.08 & 6.35 \\
			& & & 25\% label noise   & 47.6 & 5.15\\
			& & & 25\% feature noise & 30.37 & 7.13 \\
			\midrule
			\multirow{3}*{\autoref{tab:res classification Gaussian F1}} &
			\multirow{3}*{Gaussian} &   
			\multirow{3}*{$F_{1}$}   
			& without noise          & 31.37 & 7.49\\
			& & & 25\% label noise   & 59.71& 27.60\\
			& & & 25\% feature noise & 32.86& 8.05\\
			\bottomrule
		\end{tabularx}
		\label{tab: FF and ChisqF classifier}
	\end{table}
        
        Next, we apply the Nemenyi post-hoc test to examine the specific distinctions among the classifiers. According to the Nemenyi test, two classifiers are considered significantly different if the difference in their average ranks exceeds the critical difference ($CD$). The $CD$ is computed as\begin{equation}CD = q_{0.1}(k)\sqrt{\frac{k(k+1)}{6N}}=2.693\times\sqrt{\frac{7\times(7+1)}{6\times15}}=2.12,\end{equation}where $q_{0.1}=2.693$. 
        We used $CD$ diagrams \autoref{fig:CD-ACC} and \autoref{fig:CD-F1} to compare the average rankings of each SVM with different kernels and noise types. The top line shows the average ranks, with colors changing from blue to black. Groups of algorithms with no significant differences are linked with a red line. 
     \begin{figure}[H]
		\centering
		\subfigure[without noise(linear)]{
			\includegraphics[scale=0.6]{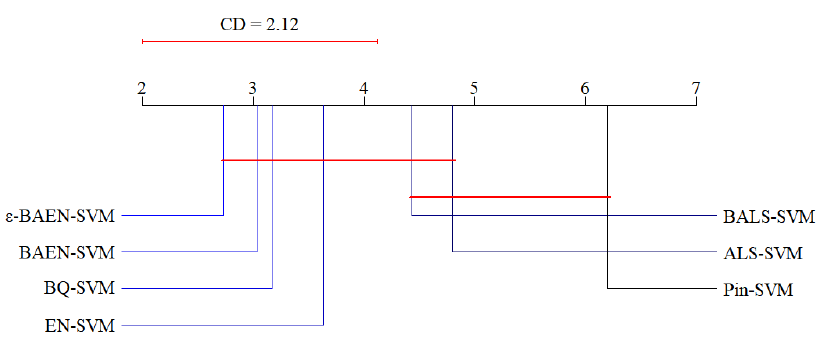}
		}
		\subfigure[25\% label noise(linear)]{
			\includegraphics[scale=0.6]{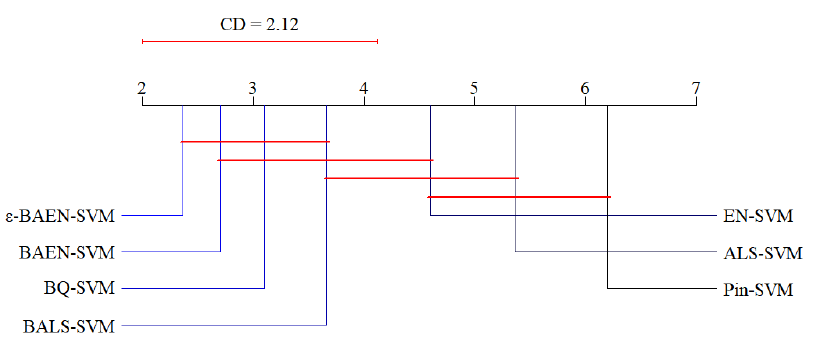}
		}
        \\
        \centering
		\subfigure[25\% feature noise(linear)]{
			\includegraphics[scale=0.6]{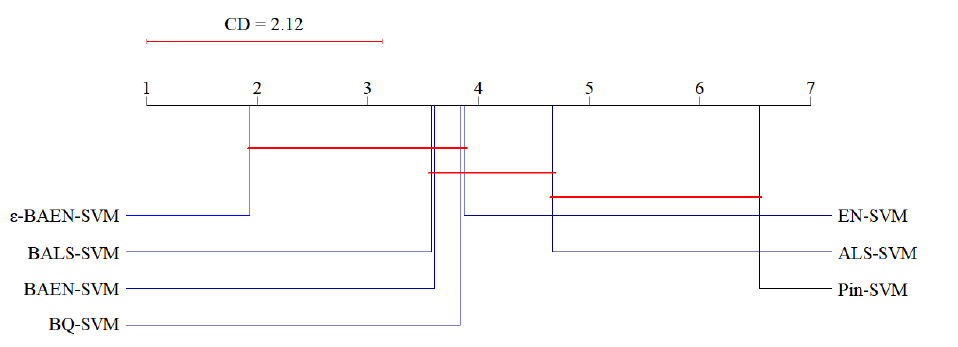}
		}
		\subfigure[without noise(RBF)]{
			\includegraphics[scale=0.6]{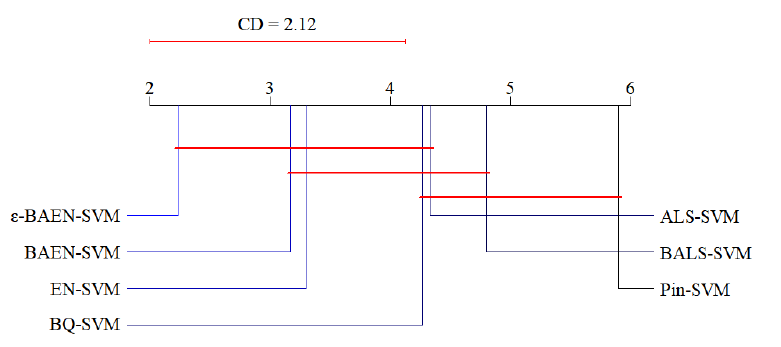}
		}
           \\
		\centering
		\subfigure[25\% label noise(RBF)]{
			\includegraphics[scale=0.6]{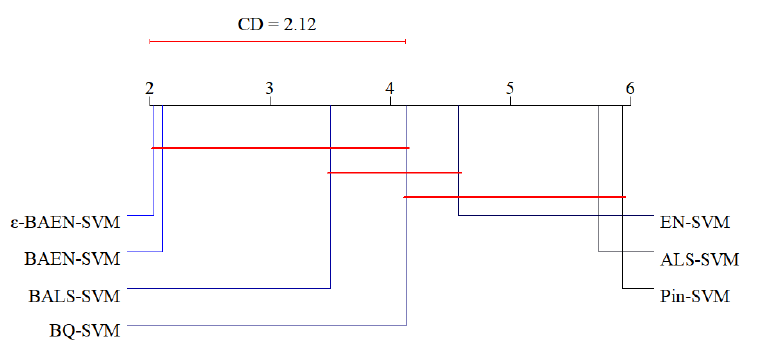}
		}
		\subfigure[25\% feature noise(RBF)]{
			\includegraphics[scale=0.6]{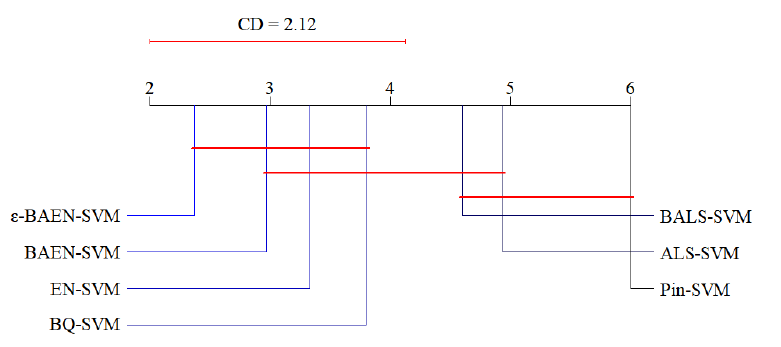}
		}
		\caption{Comparison ACC with the Nemenyi test}
		\label{fig:CD-ACC}
	\end{figure}
    
    \begin{figure}[H]
		\centering
		\subfigure[without noise(linear)]{
			\includegraphics[scale=0.6]{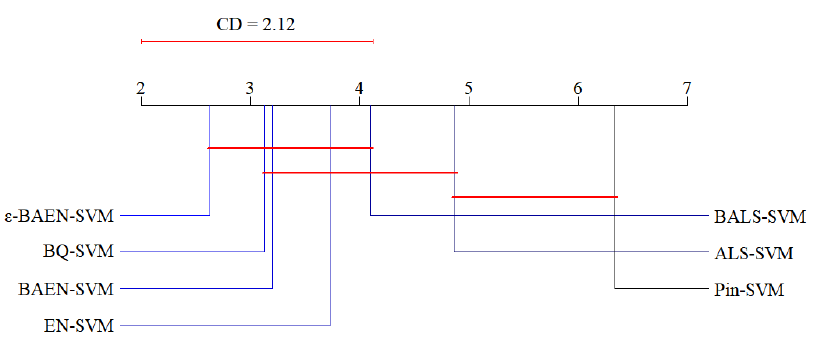}
		}
		\subfigure[25\% label noise(linear)]{
			\includegraphics[scale=0.6]{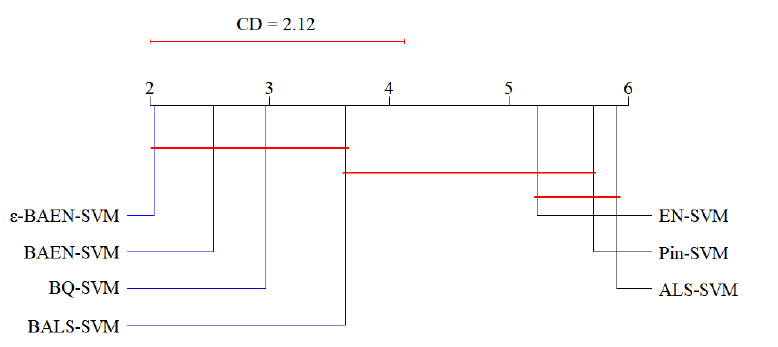}
		}
        \\
        \centering
		\subfigure[25\% feature noise(linear)]{
			\includegraphics[scale=0.6]{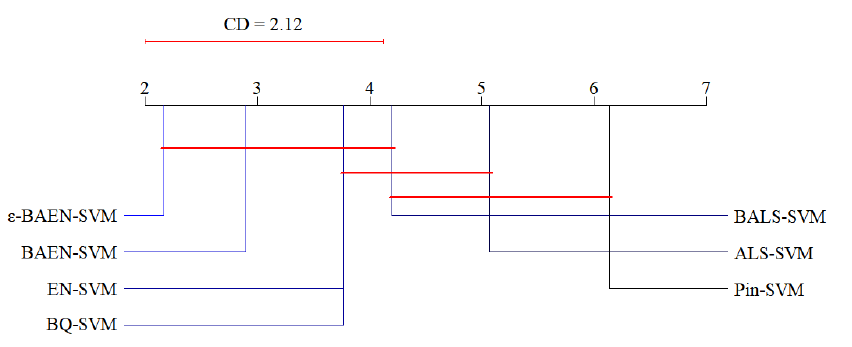}
		}
		\subfigure[without noise(RBF)]{
			\includegraphics[scale=0.6]{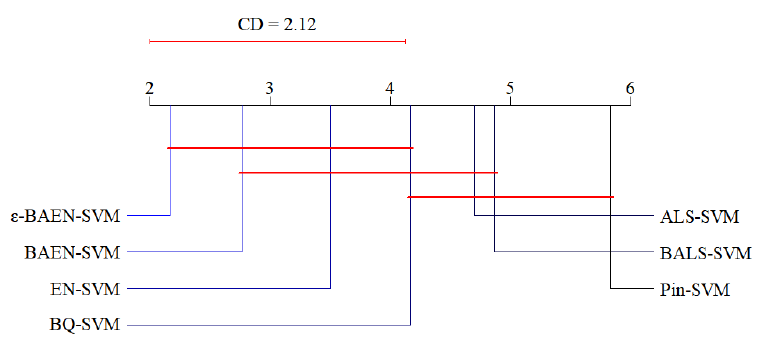}
		}
           \\
		\centering
		\subfigure[25\% label noise(RBF)]{
			\includegraphics[scale=0.6]{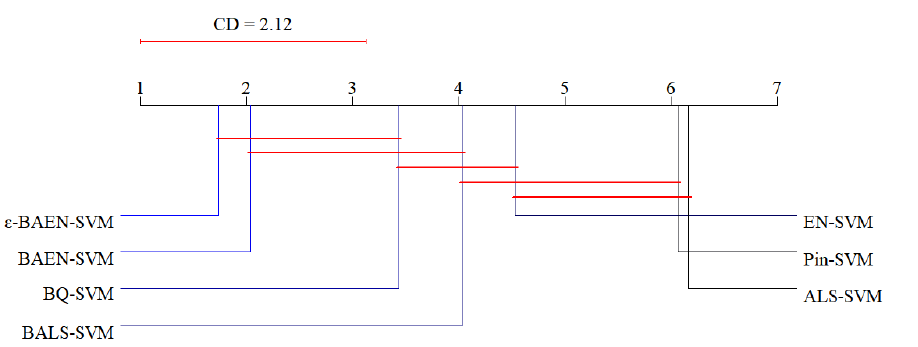}
		}
		\subfigure[25\% feature noise(RBF)]{
			\includegraphics[scale=0.6]{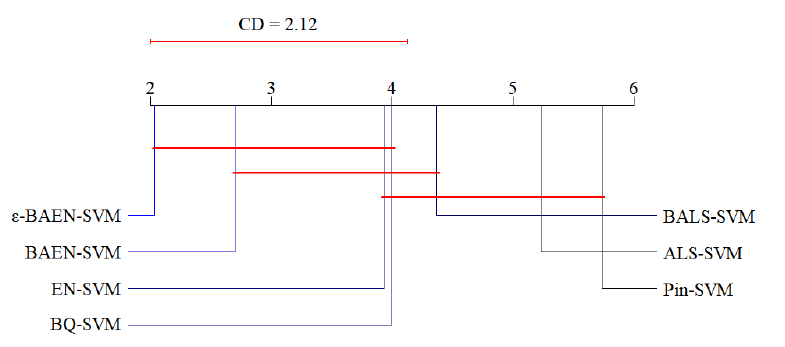}
		}
		\caption{Comparison $F_1$ with the Nemenyi test}
		\label{fig:CD-F1}
	\end{figure}
    As shown in  \autoref{fig:CD-ACC}, the $\varepsilon$-BAEN-SVM outperforms all other SVM models under the ACC evaluation metric. Its advantage becomes particularly pronounced when dealing with 25\% label noise and feature noise.
    Regarding label noise, \autoref{fig:CD-ACC}(b) and \autoref{fig:CD-F1}(e)show that both $\varepsilon$-BAEN-SVM and BAEN-SVM have comparable performance. They significantly surpass EN-SVM, Pin-SVM, and ALS-SVM. Notably, the significant difference between $\varepsilon$-BAEN-SVM and EN-SVM indicates that $\varepsilon$-BAEN-SVM effectively addresses EN-SVM's high sensitivity to label noise. When faced with 25\% feature noise in \autoref{fig:CD-ACC}(c) and \autoref{fig:CD-F1}(f), $\varepsilon$-BAEN-SVM shows even greater superiority, especially under linear kernel conditions. \autoref{fig:CD-F1}(a) and \autoref{fig:CD-F1}(f), reveal that under Gaussian kernel settings, $\varepsilon$-BAEN-SVM consistently achieves higher average ranks than under linear kernels. This highlights its distinct advantages.

    \section{Conclusion}\label{sec: Conclusion}
        This paper addresses the problems of existing support vector machines (SVMs), including low robustness to noise and lack of sparsity. Based on the $\varepsilon$-insensitive asymmetric elastic net loss and the RML framework, we propose a robust and sparse SVM model called $\varepsilon$-BAEN-SVM. Theoretical analysis shows that  $L_{\text{baen}}^{\varepsilon}$ is sparser than $L_{\text{baen}}$ and $L_{\text{aen}}$. In addition, $L_{\text{baen}}^{\varepsilon}$ has boundedness and asymmetry. Further theoretical analysis proves that $\varepsilon$-BAEN-SVM is insensitive to noise. This ensures good robustness in practical applications. To solve the non-convex problem of the nonlinear $\varepsilon$-BAEN-SVM, we design the HQ-ClipDCD algorithm. This algorithm transforms the original non-convex problem into a convex subproblem. The subproblem is a weighted $\varepsilon$-insensitive SVM with an asymmetric elastic net loss. This transformation provides a clear explanation for the model's robustness. Experimental results on simulated and real benchmark datasets show that $\varepsilon$-BAEN-SVM achieves better classification performance than competing models on both clean and noise-corrupted data. Statistical tests further confirm its superiority and robustness.
        
        The proposed $\varepsilon$-BAEN-SVM makes effective progress in improving model robustness and sparsity. However, some issues deserve further study. On small and medium-sized datasets, $\varepsilon$-BAEN-SVM with the HQ-ClipDCD algorithm already shows good classification performance and robustness. Nevertheless, each iteration requires solving a quadratic programming problem. This limits its application to large-scale datasets. Future work will focus on introducing low-rank approximation techniques for the kernel matrix. These techniques can reduce kernel computation and storage costs, thereby improving the scalability of the nonlinear $\varepsilon$-BAEN-SVM in big data scenarios. In addition, given the robust performance of $\varepsilon$-BAEN-SVM in noisy environments, we plan to apply it to high-risk fields with uneven data quality, such as medical diagnosis and financial fraud detection.

    \bibliography{main}

        \end{document}